\documentclass[english]{article}
\usepackage{lmodern}

\usepackage[T1]{fontenc}
\usepackage[a4paper]{geometry}
\usepackage{dsfont}
\usepackage{color}
\usepackage{babel}
\usepackage{amsmath}
\usepackage{amsthm}
\usepackage{amssymb}
\usepackage{stmaryrd}
\usepackage{graphicx}
\usepackage{setspace}
\usepackage{esint}
\usepackage{caption}
\newcommand*\rot{\rotatebox{90}}
\usepackage[authoryear]{natbib}
\usepackage[unicode=true,bookmarks=true,bookmarksnumbered=false,bookmarksopen=false,breaklinks=false,pdfborder={0 0 0},pdfborderstyle={},backref=page,colorlinks=true]{hyperref}
\hypersetup{pdftitle={VBPI-SIBranch},pdfauthor={Tianyu Xie, Frederick A. Matsen IV, Marc A. Suchard, Cheng Zhang},linkcolor=RoyalBlue,citecolor=RoyalBlue}
\usepackage{float}
\usepackage{subfig}
\usepackage[dvipsnames,svgnames,x11names,hyperref]{xcolor}
\usepackage{nicefrac}

\usepackage{epstopdf}
\usepackage{natbib}
\usepackage{url} 
\usepackage{caption}
\usepackage{multirow}
\usepackage{amssymb}
\usepackage{multicol}
\usepackage{booktabs}
\usepackage{tikz}
\usepackage{bbm}
\usepackage{bm}
\usetikzlibrary{positioning, bayesnet}
\usepackage{algorithm2e}
\RestyleAlgo{ruled}

\newtheorem{theorem}{Theorem}

\newtheorem{definition}{Definition}
\newtheorem{proposition}{Proposition}

\usepackage[textsize=tiny]{todonotes}

\makeatletter
\newcommand{\myfnsymbol}[1]{%
  \expandafter\@myfnsymbol\csname c@#1\endcsname
}
\newcommand{\@myfnsymbol}[1]{%
  \ifcase #1
  \or 1
  \or 2
  \or 3
  \or 4
  \or \TextOrMath{\textasteriskcentered}{*}
  \or \TextOrMath{\textdagger}{\dagger}
  \fi
}
\newcommand{\affiliationA}{\@myfnsymbol{1}}
\newcommand{\affiliationB}{\@myfnsymbol{2}}
\newcommand{\affiliationC}{\@myfnsymbol{3}}
\newcommand{\affiliationD}{\@myfnsymbol{4}}
\newcommand{\correspondingA}{\@myfnsymbol{5}}
\makeatother

\begin{document}

\title{Variational Bayesian Phylogenetic Inference with Semi-implicit Branch Length Distributions}

\author{
Tianyu Xie\textsuperscript{\affiliationA},
Frederick A. Matsen IV\textsuperscript{\affiliationB},
   Marc A. Suchard\textsuperscript{\affiliationC},
Cheng Zhang\textsuperscript{\affiliationD,\correspondingA}
}

\date{
}

\renewcommand{\thefootnote}{\myfnsymbol{footnote}}
\maketitle
\footnotetext[1]{School of Mathematical Sciences, Peking University,
   Beijing, 100871, China. Email: tianyuxie@pku.edu.cn}%
\footnotetext[2]{Computational Biology Program, Fred Hutchinson Cancer Research Center,
   Department of Genome Sciences and Department of Statistics,
   University of Washington,
   Seattle, WA 98195, USA. Email: matsen@fredhutch.org}%
\footnotetext[3]{Department of Biostatistics, Department of Biomathematics, and Department of Human Genetics, University of California, Los Angeles, CA 90095, USA. Email: msuchard@ucla.edu}%
\footnotetext[4]{School of Mathematical Sciences and Center for Statistical Science, Peking University, Beijing, 100871, China. Email: chengzhang@math.pku.edu.cn}
\footnotetext[5]{Corresponding author}%

\setcounter{footnote}{0}
\renewcommand{\thefootnote}{\fnsymbol{footnote}}

\begin{abstract}
Reconstructing the evolutionary history relating a collection of molecular sequences is the main subject of modern Bayesian phylogenetic inference.
However, the commonly used Markov chain Monte Carlo methods can be inefficient due to the complicated space of phylogenetic trees, especially when the number of sequences is large.
An alternative approach is variational Bayesian phylogenetic inference (VBPI) which transforms the inference problem into an optimization problem.
While effective, the default diagonal lognormal approximation for the branch lengths of the tree used in VBPI is often insufficient to capture the complexity of the exact posterior.
In this work, we propose a more flexible family of branch length variational posteriors based on semi-implicit hierarchical distributions using graph neural networks.
We show that this semi-implicit construction emits straightforward permutation equivariant distributions, and therefore can handle the non-Euclidean branch length space across different tree topologies with ease.
To deal with the intractable marginal probability of semi-implicit variational distributions, we develop several alternative lower bounds for stochastic optimization.
We demonstrate the effectiveness of our proposed method over baseline methods on benchmark data examples, in terms of both marginal likelihood estimation and branch length posterior approximation.
\end{abstract}


\section{Introduction}\label{sec:introduction}
Bayesian phylogenetic inference is a fundamental statistical framework in molecular evolution and systematics that aims to reconstruct the evolutionary histories among taxa or other biological entities, with a wide range of applications including genomic epidemiology \citep{Du_Plessis2021-tq} and conservation genetics \citep{DeSalle2004-mm}.
Given observed biological sequences (e.g., DNA, RNA, protein) and a model of molecular evolution, Bayesian phylogenetic inference seeks to estimate the posterior distribution of phylogenetic trees.
The exact computation of this posterior is intractable as it would require integrating out all possible tree topologies and branch lengths.
Thus, practitioners use approximation methods.
A typical approach is Markov chain Monte Carlo (MCMC) \citep{yang1997bayesian, Mau99, Larget1999MarkovCM, ronquist2012mrbayes} that relies on efficient proposal mechanisms to explore the tree space.
As the tree space, however, contains both continuous and discrete components (e.g., the branch lengths and the tree topologies),  phylogenetic posteriors are often complex multimodal distributions. 
Further, tree proposals used in MCMC are often limited to local modifications that lead to low exploration efficiency, and this makes Bayesian phylogenetic inference a challenging task for MCMC algorithms \citep{Lakner08, Hhna2012-pm, Whidden2014QuantifyingME, Dinh2017-oj, hassler2023data}.

An alternative approximate Bayesian inference method is variational inference (VI) \citep{jordan1999introduction, Blei2016VariationalIA}. 
Unlike MCMC, VI seeks the closest member from a family of candidate distributions (i.e., the variational family) to the posterior distribution by minimizing some statistical distance criterion, usually the Kullback-Leibler (KL) divergence.
By converting the inference problem into an optimization problem, VI tends to be faster and easier to scale up to large data \citep{Blei2016VariationalIA}.
Unlike MCMC methods that are asymptotically unbiased, variational approximations are often biased, especially when the variational family of distributions is insufficiently flexible.
The success of VI, therefore, relies on the construction of expressive variational families and efficient optimization procedures.
While classical mean-field VI requires conditionally conjugate models and often suffers from limited approximation power, much progress has been made in recent years to allow for more generic model-agnostic optimization methods \citep{ranganath2014bbvi} and more flexible variational families that have tractable densities \citep{rezende2015NF, dinh2016realnvp, kingma2016IAF, papamakarios2021nf}.
Moreover, variational families can be further expanded by allowing implicit models which have intractable densities but are easy to sample from \citep{Huszar17}.
These implicit models are usually constructed by either pushing forward a simple base distribution through a parameterized map, i.e., deep neural networks \citep{Tran17, AVB, Shi18, SSM} or using a semi-implicit hierarchical architecture \citep{yin2018sivi,titsias2019uivi,Sobolev2019IWHVI}.

Until recently, VI has received limited attention in the field of phylogenetics.
For a fixed tree topology, VI-based approaches have been developed to approximate the posterior of continuous parameters via coordinate ascent \citep{Dang19} and to estimate marginal likelihoods for model comparison \citep{Fourment19}.
However, when taking the tree topology as also random, the design of variational methods can be highly nontrivial, partially due to the absence of an appropriate family of distributions on phylogenetic trees.
\citet{Zhang2019VariationalBP} took the first step in this direction by developing a general framework for variational Bayesian phylogenetics inference (VBPI), where they used a product of a tree topology model and a branch length model to provide variational approximations.
They originally chose the tree topology model to be a subsplit Bayesian network (SBN), a powerful probabilistic graphical model specifically designed for distributions over tree topologies.
Although effective, SBNs require a pre-selected sample of candidate tree topologies that confines their support to a subspace of all possible tree topologies.
Many other approaches have been introduced recently \citep{koptagel2022vaiphy, xie2023artree, mimori2023geophy, Zhou2023PhyloGFN} that remove this constraint and hence may provide more flexible distributions over the entire tree topology space.
The conditional branch length model is often a simple diagonal lognormal distribution that is amortized over tree topologies via either hand-engineered heuristic features \citep{Zhang2019VariationalBP} or learnable topological features \citep{Zhang2023learnable}.
Although there were follow-up works for improved branch length models, e.g., VBPI with normalizing flows \citep{Zhang2020ImprovedVB}, the requirement of permutation equivariant transformations and explicit density adds to the difficulty of architecture design and may also limit the approximation accuracy, especially for complicated real data branch length posteriors.

In this work, we introduce a semi-implicit hierarchical construction for the branch length model in VBPI, with an emphasis on unrooted models.
We show that distributions under this construction can be easily made permutation invariant; therefore, they are naturally suitable for modeling branch lengths across different tree topologies.
To address the intractable density of semi-implicit variational distributions, we adapt ideas from semi-implicit variational inference (SIVI) \citep{yin2018sivi} and importance weighted hierarchical variational inference (IWHVI) \citep{Sobolev2019IWHVI} to design alternative surrogate objectives for optimization.
Our synthetic and real-world experiments show that VBPI with semi-implicit branch length distributions (VBPI-SIBranch) outperforms baseline methods in both marginal likelihood estimation and branch length posterior approximation.

The rest of the paper is organized as follows.
In Section \ref{sec:background}, we introduce the essential ingredients of SIVI methods, phylogenetic models, and the variational Bayesian phylogenetic inference framework. 
In Section \ref{sec:method}, we present our semi-implicit branch length variational distributions, describe two surrogate objective functions for optimization, and prove their statistical properties.
In Section \ref{sec:experiments}, we conduct experiments to compare VBPI-SIBranch to baseline methods in terms of both marginal likelihood estimation and branch length approximation.
We conclude with a discussion in Section \ref{sec:conclusion}.
\section{Background}\label{sec:background}

\subsection{Semi-implicit Variational Inference}
Given observed data $\mathcal{D}$ and random variables $\bm{x}$ that characterize the generation of $\mathcal{D}$, VI reformulates the Bayesian inference of a posterior distribution $P(\bm{x}|\mathcal{D})\propto P(\bm{x},\mathcal{D})$ as an optimization problem by minimizing the distance between $P(\bm{x}|\mathcal{D})$ and a parametrized variational distribution $Q_{\bm{\theta}}(\bm{x})$ which is commonly assumed to have tractable density \citep{jordan1999introduction, Blei2016VariationalIA}. 
The most commonly used distance is the reversed KL divergence defined as $D_{\mathrm{KL}}\left(Q_{\bm{\theta}}(\bm{x})\|P(\bm{x}|\mathcal{D})\right)=\mathbb{E}_{Q_{\bm{\theta}}(\bm{x})}\left[\log Q_{\bm{\theta}}(\bm{x})-\log P(\bm{x}|\mathcal{D})\right]$.
As the posterior distribution $P(\bm{x}|\mathcal{D})$ is often only known up to a constant $P(\mathcal{D})$, in practice we maximize the \emph{evidence lower bound} (ELBO) instead, defined as 
\begin{equation}
L(\bm{\theta}) = \mathbb{E}_{Q_{\bm{\theta}}(\bm{x})}\log\left(
\frac{P(\bm{x},\mathcal{D})}{Q_{\bm{\theta}}(\bm{x})}
\right) =\log P(\mathcal{D}) - D_{\mathrm{KL}}\left(Q_{\bm{\theta}}(\bm{x})\|P(\bm{x}|\mathcal{D})\right)\leq \log P(\mathcal{D}).
\end{equation}
Another popular objective function for VI is the \emph{multi-sample lower bound} \citep{burda2015iwae, Mnih2016VariationalIF}
\begin{equation}
L^K(\bm{\theta}) = \mathbb{E}_{Q_{\bm{\theta}}(\bm{x}^{1:K})}\log\left(\frac{1}{K}\sum_{k=1}^K
\frac{P(\bm{x}^k,\mathcal{D})}{Q_{\bm{\theta}}(\bm{x}^k)}
\right) \leq \log P(\mathcal{D}),
\end{equation}
where one averages over multiple samples with $Q_{\bm{\theta}}(\bm{x}^{1:K})=\prod_{k=1}^KQ_{\bm{\theta}}(\bm{x}^k)$, and we will use $K$ for the number of particles in the rest of the paper.

Beyond the explicit assumptions of $Q_{\bm{\theta}}(\bm{x})$,
\textit{semi-implicit variational inference} (SIVI) \citep{yin2018sivi} assumes a more flexible variational family defined hierarchically as
\begin{equation}\label{eq:sivi-family}
Q_{\bm{\theta}}(\bm{x}) = \int Q_{\bm{\theta}}(\bm{x}|\bm{z})Q_{\bm{\theta}}(\bm{z})\mathrm{d}\bm{z},
\end{equation}
where $\bm{z}$ is a latent variable, $Q_{\bm{\theta}}(\bm{x}|\bm{z})$ is required to be explicit and $Q_{\bm{\theta}}(\bm{z})$ can be implicit.
Compared to standard VI, the above semi-implicit hierarchical construction allows a much richer family that can capture complicated correlation between parameters \citep{yin2018sivi}.
However, the ELBO $L(\bm{\theta})$ used in standard VI is no longer suitable for SIVI as $Q_{\bm{\theta}}(\bm{x})$ is intractable.
The variational family (\ref{eq:sivi-family}) is instead fitted by maximizing the \textit{semi-implicit lower bound} (SILB; \citet[equation 9]{yin2018sivi})
\begin{equation}\label{eq:sivi-surrogate}
\mathbb{E}_{Q_{\bm{\theta}}(\bm{x},\bm{z}^{0})Q_{\bm{\theta}}(\bm{z}^{1:J})}\log \left(\frac{P(\bm{x},D)}{\frac{1}{J+1}\sum_{j=0}^J Q_{\bm{\theta}}(\bm{x}|\bm{z}^j)}\right),
\end{equation}
where $Q_{\bm{\theta}}(\bm{x},\bm{z}^{0})=Q_{\bm{\theta}}(\bm{x}|\bm{z}^{0})Q_{\bm{\theta}}(\bm{z}^{0})$, $Q_{\bm{\theta}}(\bm{z}^{1:J})=\prod_{j=1}^J Q_{\bm{\theta}}(\bm{z}^{j})$, and $J$ is the number of extra samples for an importance-sampling-based estimator of $Q_{\bm{\theta}}(\bm{x})$. 
Here, we put $\bm{z}^0$ together with $\bm{x}$ to emphasize that $\bm{x}$ depends on $\bm{z}^{0}$.
Noticing that samples from $Q_{\bm{\theta}}(\bm{z})$ might not be informative for estimating $Q_{\bm{\theta}}(\bm{x})$, 
\citet{Sobolev2019IWHVI} use an auxiliary reverse model $R_{\bm{\alpha}}(\bm{z}|\bm{x})$ as the importance distribution, and maximize the following \textit{importance weighted lower bound} (IWLB; \citet[equation 4]{Sobolev2019IWHVI})
\begin{equation}\label{eq:iwhvi-surrogate}
\mathbb{E}_{Q_{\bm{\theta}}(\bm{x},\bm{z}^{0})R_{\bm{\alpha}}(\bm{z}^{1:J}|\bm{x})}\log \left(\frac{P(\bm{x},D)}{\frac{1}{J+1}\sum_{j=0}^J \frac{Q_{\bm{\theta}}(\bm{x},\bm{z}^j)}{R_{\bm{\alpha}}(\bm{z}^{j}|\bm{x})}}\right),
\end{equation}
where $R_{\bm{\alpha}}(\bm{z}^{1:J}|\bm{x})=\prod_{j=1}^J R_{\bm{\alpha}}(\bm{z}^{j}|\bm{x})$. Here, $Q_{\bm{\theta}}(\bm{z})$ and $R_{\bm{\alpha}}(\bm{z}|\bm{x})$ need to be explicit. 

\subsection{Phylogenetic Trees}
Given $N$ observed taxa, an important goal in phylogenetic inference is to estimate their evolutionary history, which is often described as a \emph{phylogenetic tree} that includes a tree topology $\tau$ and a vector of non-negative branch lengths $\bm{q}$ for the edges on $\tau$. 

The \textit{tree topology} $\tau$ is a bifurcating tree graph with a node set $V(\tau)$ and an edge set $E(\tau)$. 
There are two types of nodes in $V(\tau)$: nodes with degree one are called \textit{leaf nodes} that represent the existing (observed) taxa; nodes with degree two or three are called \textit{internal nodes} that represent the ancient (unobserved) taxa.
For a rooted tree topology, there exists a unique node with degree two called the \textit{root node} (or the root for simplicity), and the edges in $E(\tau)$ are directed away from the root node. 
For an unrooted tree topology, all the nodes in $V(\tau)$ have one or three degrees, and all the edges in $E(\tau)$ are undirected.
Furthermore, an unrooted tree topology can be converted to a rooted one (and vice versa) by placing a root node on an edge (removing the root node and connecting its two neighbors).
As mentioned above, the focus of our work is on unrooted phylogenetic trees (rather than rooted phylogenetic time trees), however, our algorithm can be easily adapted to rooted phylogenetic trees.
In this article, we use ``tree topology'' for an unrooted tree topology unless otherwise specified.

For each edge $e=(u,v)\in E(\tau)$, there is a non-negative scalar $q_{uv}$ (or equivalently, $q_e$) called the \textit{branch length}. 
The branch length $q_{uv}$ quantifies the amount of evolution along edge $e=(u,v)$, i.e., the expected number of character substitutions between the two neighboring nodes $u$ and $v$.
The vector $\bm{q}=[q_{e}]_{e\in E(\tau)}$ contains all the branch lengths associated with tree topology $\tau$.

\subsection{Bayesian Phylogenetic Inference}
The leaf nodes of a phylogenetic tree correspond to the observed taxa, whose aligned molecular sequences
is represented as a matrix $\bm{Y}=\{\bm{Y}_1,\bm{Y}_2,\ldots,\bm{Y}_N\}\in \Omega^{N\times S}$.
Here, $\Omega$ is the alphabet set of characters (e.g., nucleotides: A, C, G, T) that comprise the sequences, and $S$ is the character sequence length.
For each $1\leq s\leq S$, $\bm{Y}_s$ denotes the observed characters of all taxa at a single aligned position, also called a site, that are homologous, meaning that they all arose from a common character somewhere on the phylogenetic tree through a process of replication and substitution along its edges.
The goal of phylogenetic inference is then to reconstruct $(\tau, \bm{q})$ based on the observed sequence data $\bm{Y}$. 

Given a rooted tree topology $\tau$ and branch lengths $\bm{q}$, the generative process of the observed data $\bm{Y}$ can be described as follows.
Starting from the root node, the evolution along the edges of the tree is governed by a substitution model, often a continuous-time Markov chain (CTMC) that governs the transition probabilities among the characters from a parent node to its child node \citep{jukes1969evolution, tavare1986some}.
Let $\bm{Q}$ be the transition rate matrix.
The transition probability along an edge $(u,v)$ at site $s$ is $P_{a^s_u a^s_v}(q_{uv}) = \left(\exp(q_{uv}\bm{Q})\right)_{a^s_u, a^s_v}$, where $a_{u}^s$ is the character assignment of node $u$ at site $s$.
Assuming that each site evolves independently and identically, the \textit{phylogenetic likelihood} of observing $\bm{Y}$ is obtained by summing out all the possible states of internal nodes as
\begin{equation}\label{eq-likelihood-Y}
P(\bm{Y}|\tau,\bm{q}) = \prod_{s=1}^S P(\bm{Y}_s|\tau,\bm{q}) = \prod_{s=1}^S \sum_{a^s}\eta(a^s_r)\prod_{(u,v)\in E(\tau)}P_{a^s_u a^s_v}(q_{uv}),
\end{equation}
where $r$ represents the root node, $a^s$ ranges over all extensions of $\bm{Y}_s$ to the internal nodes, and $\eta$ is a prior distribution on the root states. 
The phylogenetic likelihood (\ref{eq-likelihood-Y}) can be efficiently evaluated by Felsenstein's pruning algorithm \citep{felsenstein2004inferring}.

For an unrooted tree topology, one can also obtain a valid phylogenetic likelihood from equation (\ref{eq-likelihood-Y}) by placing a root node $r$ on an arbitrary edge at any position. 
In fact, equation (\ref{eq-likelihood-Y}) does not depend on the location of the root node as long as the CTMC is time-reversible and one assumes that the root prior is the stationary distribution of $\bm{Q}$ \citep{Felsenstein81}.
This is also a common choice of $\eta$ in practice.

Given a prior distribution $P(\tau,\bm{q})$ over the space of phylogenetic trees, the joint posterior density takes the following form
\begin{equation}\label{eq-posterior}
P(\tau,\bm{q}|\bm{Y}) = \frac{P(\bm{Y}|\tau, \bm{q})P(\tau,\bm{q})}{P(\bm{Y})}\propto P(\bm{Y}|\tau, \bm{q})P(\tau,\bm{q}).
\end{equation}
A common choice of the prior consists of a uniform distribution over tree topologies and independent exponential distributions over branch lengths \citep{ronquist2012mrbayes}.  

\subsection{Variational Bayesian Phylogenetic Inference}\label{sec:vbpi}
To estimate the phylogenetic posterior in VI, VBPI posits a parameterized variational family $Q_{\bm{\phi},\bm{\psi}}(\tau,\bm{q})$ that is a product of a tree topology model $Q_{\bm{\phi}}(\tau)$ and a branch length model $Q_{\bm{\psi}}(\bm{q}|\tau)$.
The variational approximation is then obtained by maximizing the multi-sample lower bound (MLB)
\begin{equation}\label{eq-lower-bound}
L^{K}(\bm{\phi},\bm{\psi}) = \mathbb{E}_{Q_{\bm{\phi},\bm{\psi}}(\tau^{1:K},\bm{q}^{1:K})}\log \left(\frac{1}{K}\sum_{k=1}^K\frac{P(\bm{Y}|\tau^k,\bm{q}^k) P(\tau^k, \bm{q}^k)}{Q_{\bm{\phi}}(\tau^k)Q_{\bm{\psi}}(\bm{q}^k|\tau^k)}\right),
\end{equation}
where $Q_{\bm{\phi},\bm{\psi}}(\tau^{1:K},\bm{q}^{1:K})\equiv\prod_{k=1}^K Q_{\bm{\phi},\bm{\psi}}(\tau^k,\bm{q}^k)$.
The optimization of equation (\ref{eq-lower-bound}) is done through stochastic gradient ascent (SGA), where the stochastic gradients for tree topology parameters and branch length parameters are obtained by the VIMCO/RWS estimator \citep{Mnih2016VariationalIF, RWS} and the reparameterization trick \citep{VAE} respectively.
Compared to the ELBO, the MLB (\ref{eq-lower-bound}) enables efficient variance-reduced gradient estimators and encourages exploration over the vast and multimodal tree space.
However, as a large $K$ may also reduce the signal-to-noise ratio and deteriorate the training of variational parameters \citep{Rainforth19}, a moderate $K$ is suggested in practice \citep{Zhang22VBPI}.

The tree topology model $Q_{\bm{\psi}}(\tau)$ can be parametrized by SBNs \citep{Zhang2018GeneralizingTP} as follows.
A non-empty subset of the leaf nodes is called a \textit{clade} with a total order $\succ$ (e.g., lexicographical order) on all clades.
An ordered clade pair $(W,Z)$ satisfying $W\cap Z=\emptyset$ and $W\succ Z$ is called a \textit{subsplit}.
An SBN is then defined as a Bayesian network whose nodes take subsplit values or singleton clade values that describe the local topological structures of tree topologies.
For a rooted tree topology, one can find its corresponding node assignment of SBNs by starting from the root node, iterating towards the leaf nodes, and gathering all the visited parent-child subplit pairs.
The SBN-based probability of a rooted tree topology $\tau$ then takes the form
\begin{equation}
p_{\mathrm{sbn}}(T=\tau) = p(S_1=s_1)\prod_{i>1}p(S_i=s_i|S_{\pi_i} = s_{\pi_i}),
\end{equation}
where $S_i$ denotes the subsplit- or singleton-clade-valued random varaibles at node $i$ (node 1 is the root node), $\pi_i$ is the index set of the parents of node $i$ and $\{s_{i}\}_{i\geq 1}$ is the corresponding node assignment.
For unrooted tree topologies, we can also define their SBN-based probabilities by viewing them as rooted tree topologies with unobserved roots and summing out the root positions.
For VBPI, the conditional probabilities in SBNs are often parameterized based on a subsplit support estimated from fast bootstrap or MCMC tree samples \citep{Minh2013UltrafastAF, Zhang22VBPI}.
See more details of SBNs in Appendix \ref{details-sbns}.

As the branch lengths are non-negative, the branch length model $Q_{\bm{\psi}}(\bm{q}|\tau)$ is often taken to be a diagonal lognormal distribution
\begin{equation}\label{eq-branch-lognormal}
Q_{\bm{\psi}}(\bm{q}|\tau) = \prod_{e\in E(\tau)} p^{\textrm{Lognormal}}\left(q_e\,|\,\mu(e,\tau),\sigma(e,\tau)\right),
\end{equation}
where $\mu(e,\tau)$ and $\sigma(e,\tau)$ are the mean and standard deviation parameters of the lognormal distribution, and are amortized over the tree topologies via shared local structures \citep{Zhang2019VariationalBP} or learnable node features \citep{Zhang2023learnable}.
However, the simple diagonal lognormal variational approximation (\ref{eq-branch-lognormal}) maybe too simple to capture the complicated posterior distributions of branch lengths due to the hierarchical structure of tree topologies.
Although \citet{Zhang2020ImprovedVB} proposed to parameterize $Q_{\bm{\psi}}(\bm{q}|\tau)$ with normalizing flows (VBPI-NF), the requirement of invariant and explicit distribution confines the flexibility of these branch length models.
\section{Methodology}\label{sec:method}
In this section, we present a more flexible family of branch length distributions for VBPI, featuring a hierarchical semi-implicit structure, which we call VBPI-SIBranch\@.
We begin by outlining the construction of semi-implicit branch length distributions with learnable topological features via powerful graph neural networks (GNNs) \citep{kipf2017GCN, Gilmer2017NeuralMP}.
These distributions exhibit natural permutation equivariance, making them well-suited for branch length approximation across various tree topologies.
Note that the branch lengths are defined upon the edges and thus do not naturally map across different tree topologies.
We then develop efficient surrogate objective functions, provide theoretical guarantees, and illustrate their application in the training process.

\begin{figure}
    \centering
    \includegraphics[width=\linewidth]{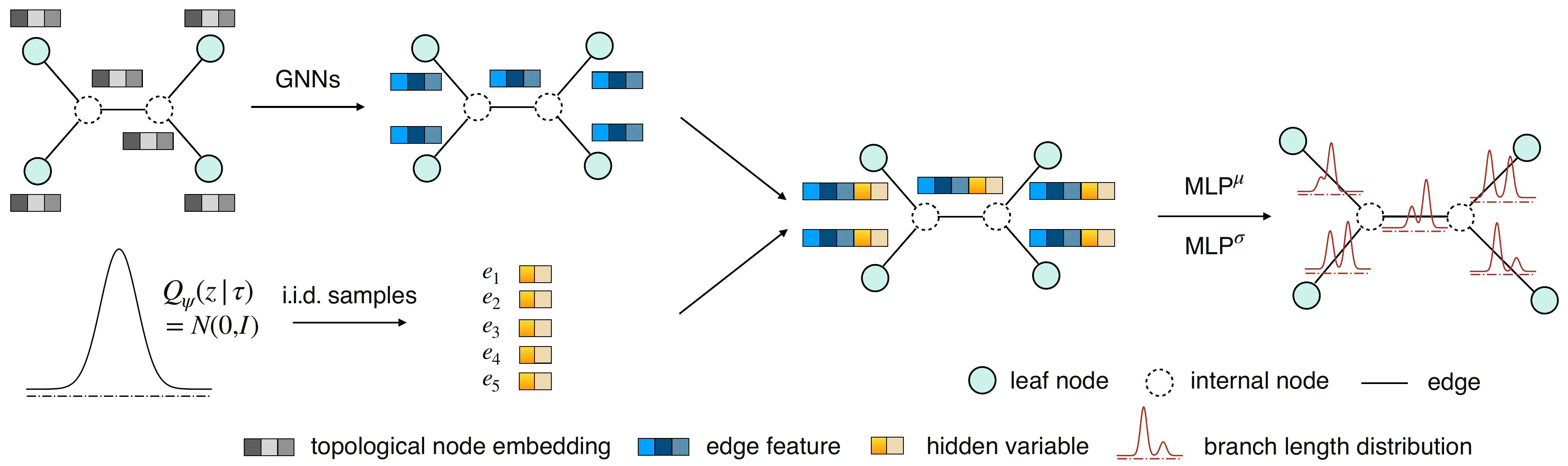}
    \caption{An overview of VBPI-SIBranch for a five-leaf phylogenetic tree.
    We begin with topological node embeddings \citep{Zhang2023learnable} (upper left) and apply GNNs to obtain the edge features.
    These features, joined together with the i.i.d. hidden variables, are finally fed into the $\mathrm{MLP}^\mu$ and $\mathrm{MLP}^\sigma$ to form the parameters of branch length distributions.
    }
    \label{fig:overview}
\end{figure}

\subsection{Semi-implicit Branch Length Distributions}
To improve the expressiveness of branch length models, 
we introduce the following semi-implicit hierarchical construction for branch length distributions
\begin{equation}\label{eq:si-dist}
\bm{q}\sim Q_{\bm{\psi}}(\bm{q}|\tau,\bm{z}),\ \bm{z}\sim Q_{\bm{\psi}}(\bm{z}|\tau),
\end{equation}
where $\bm{z}$ is a hidden variable with prior distribution $Q_{\bm{\psi}}(\bm{z}|\tau)$ (i.e., the mixing distribution) conditioned on the tree topology $\tau$, and $Q_{\bm{\psi}}(\bm{q}|\tau,\bm{z})$ is the conditional branch length distribution.
Both $Q_{\bm{\psi}}(\bm{z}|\tau)$ and $Q_{\bm{\psi}}(\bm{q}|\tau,\bm{z})$ are assumed to be reparameterizable, while $Q_{\bm{\psi}}(\bm{z}|\tau)$ is generally implicit and $Q_{\bm{\psi}}(\bm{q}|\tau,\bm{z})$ is required to be explicit.
Integrating out the hidden variable $\bm{z}$, we have the marginal variational distribution of branch lengths
\begin{equation}\label{eq:si-marginal-dist}
Q_{\bm{\psi}}(\bm{q}|\tau) =  \int Q_{\bm{\psi}}(\bm{q}|\tau,\bm{z}) Q_{\bm{\psi}}(\bm{z}|\tau) \mathrm{d}\bm{z}.
\end{equation}
This augmented hidden variable introduces additional flexibility to the modeling of branch lengths.
Note that equation (\ref{eq:si-marginal-dist}) degenerates to the explicit branch length distribution in vanilla VBPI when the mixing distribution $\bm{z}\sim Q_{\bm{\psi}}(\bm{z}|\tau)$ collapses to a Dirac measure.

For a given tree topology $\tau$, the distribution of its associated branch lengths $\bm{q}$ should not depend on the edge orderings on $E(\tau)$. 
This naturally requires the branch length model to be permutation invariant (Definition \ref{def:permutation-invariance}).
In what follows, we show that the semi-implicit hierarchical construction \eqref{eq:si-dist} allows permutation invariant construction of the marginal branch length distributions (Proposition \ref{prop:invariant}).

\begin{definition}[Permutation Invariance]\label{def:permutation-invariance}
For a tree topology $\tau$, let $\pi: E(\tau)\to E(\tau)$ be a specific permutation function on the edges of $\tau$ 
and $\bm{q}_\pi=[q_{\pi(e)}]_{e\in E(\tau)}$.
The branch length distribution $Q_{\bm{\psi}}(\bm{q}|\tau)$ is said to be permutation invariant, if for any permutation function $\pi$, we have $
Q_{\bm{\psi}}(\bm{q}_{\pi}|\tau) = Q_{\bm{\psi}}(\bm{q}|\tau).$
\end{definition}

\begin{proposition}\label{prop:invariant}
Suppose $\bm{z}=[\bm{z}_e]_{e\in E(\tau)}$ and $\bm{z}_{\pi}=[\bm{z}_{\pi(e)}]_{e\in E(\tau)}$. If $Q_{\bm{\psi}}(\bm{q}|\tau,\bm{z})$ and $Q_{\bm{\psi}}(\bm{z}|\tau)$ in \eqref{eq:si-dist} are permutation invariant, i.e., $Q_{\bm{\psi}}(\bm{q}_{\pi}|\tau,\bm{z}_\pi)=Q_{\bm{\psi}}(\bm{q}|\tau,\bm{z})$, $Q_{\bm{\psi}}(\bm{z}_\pi|\tau)=Q_{\bm{\psi}}(\bm{z}|\tau)$, then the marginal branch length distribution $Q_{\bm{\psi}}(\bm{q}|\tau)$ is also permutation invariant.
\end{proposition}
\begin{proof}
Let $\bm{L}_\pi$ be the permutation matrix corresponding to $\pi$ with $|\mathrm{det}(\bm{L}_\pi)|=1$. 
By the permutation invariance of $Q_{\bm{\psi}}(\bm{q}|\tau,\bm{z})$, we know that $Q_{\bm{\psi}}(\bm{q}_\pi|\tau,\bm{z}_\pi)=Q_{\bm{\psi}}(\bm{q}|\tau,\bm{z})$. 
This, together with the permutation invariance of $Q_{\bm{\psi}}(\bm{z}_\pi|\tau)$, yields
\[
Q_{\bm{\psi}}(\bm{q}_\pi|\tau) =  \int Q_{\bm{\psi}}(\bm{q}_\pi|\tau,\bm{z}_\pi) Q_{\bm{\psi}}(\bm{z}_\pi|\tau) \mathrm{d}\bm{z}_\pi = \int Q_{\bm{\psi}}(\bm{q}|\tau,\bm{z}) Q_{\bm{\psi}}(\bm{z}|\tau)|\mathrm{det}(\bm{L}_\pi)| \mathrm{d}\bm{z}=Q_{\bm{\psi}}(\bm{q}|\tau),
\]
which implies that $Q_{\bm{\psi}}(\bm{q}|\tau)$ is a permutation invariant distribution.
\end{proof}

\subsection{Graph Neural Networks for Semi-implicit Branch Length Distributions}\label{sec:embedding}
Both the invariant conditional branch length distribution $Q_{\bm{\psi}}(\bm{q}|\tau,\bm{z})$ and the mixing distribution $Q_{\bm{\psi}}(\bm{z}|\tau)$ can be parametrized by GNNs.
We will first introduce the topological node embeddings and then give a concrete example for constructing semi-implicit branch length distributions with GNNs.
\paragraph{Topological Node Embeddings}
\cite{Zhang2023learnable} introduces topological node embedding for phylogenetic trees that allows integration of deep learning methods for structural representation learning of phylogenetic trees for downstream tasks \citep{xie2023artree}.
For a tree topology $\tau$, the set of topological node embeddings is defined as $\bm{f}(\tau)=\{\bm{f}_u\in \mathbb{R}^N; u\in V(\tau)\}$ which assigns an embedding vector for each node.
To obtain the topological node embeddings, we first assign one-hot embedding vectors to the leaf nodes and then 
compute the embedding vectors for internal nodes by minimizing the Dirichlet energy
\begin{equation}
\ell(\bm{f},\tau) := \sum_{(u,v)\in E(\tau)}||\bm{f}_u-\bm{f}_v||^2
\end{equation}
that can be analytically solved by a linear-time two-pass algorithm \citep{Zhang2023learnable}. The following theorem reveals the representation power of topological node embeddings.

\begin{theorem}[Identifiability; \citet{Zhang2023learnable}]
Let $V^o(\tau)$ be the set of internal nodes of $\tau$ and $\bm{f}^o(\tau)=\{\bm{f}_u; u\in V^o(\tau)\}$ denote the topological node embeddings of $V^o(\tau)$. 
For two tree topologies $\tau_1$ and $\tau_2$, $\tau_1=\tau_2$ if and only if $\bm{f}^o(\tau_1)=\bm{f}^o(\tau_2)$.
\end{theorem}

\paragraph{Learnable Node Features}
To form learnable node features (initialized as the topological node features $\{\bm{f}_u^{(0)};u\in V(\tau)\}$) that encode the topological information of $\tau$, we utilize GNNs with message passing steps where the node features are updated by aggregating the information from their neighborhoods in a convolutional manner \citep{Gilmer2017NeuralMP}.
Concretely, the $l$-th message passing step is implemented as
\[
\begin{array}{rcl}
\bm{m}^{(l+1)}_u &=& \mathrm{AGG}^{(l)}\left(\left\{\bm{f}_v^{(l)}; v\in\mathcal{N}(u)\right\}\right),\\
\bm{f}_u^{(l+1)} & = & \mathrm{UPDATE}^{(l)}\left(\bm{f}_u^{(l)}, \bm{m}^{(l+1)}_u\right),
\end{array}
\]
where $\mathrm{AGG}^{(l)}$ and $\mathrm{UPDATE}^{(l)}$ are the aggregation function and update function in the $l$-th step parametrized by neural networks.
After $L$ message passing steps, $\{\bm{f}^{(L)}_u; u\in V(\tau)\}$ are fed into a multi-layer perceptron (MLP), i.e.,
\[
\bm{h}_u = \mathrm{MLP}\left(\bm{f}^{(L)}_u\right),
\]
which outputs the learnable node features $\bm{h}(\tau)=\{\bm{h}_u; u\in V(\tau)\}$.

\paragraph{Semi-implicit Construction}
Let $g$ be a permutation invariant function (e.g., $\mathrm{sum}$).
We first transform the learnable node features into edge features $\{\bm{h}_e; e\in E(\tau)\}$ with $\bm{h}_e=g(\{\bm{h}_u, \bm{h}_v\})$ where $u$ and $v$ are the two neighboring nodes of $e$.
Let $\bm{z}_e$ be the corresponding hidden variable for edge $e$, and the $\bm{z}=[\bm{z}_e]_{e\in E(\tau)}$ follows the mixing distribution $Q_{\bm{\psi}}(\bm{z}|\tau)$. 
These features are then concatenated to form the mixing edge features $\{\bar{\bm{h}}_e = \bm{h}_e \| \bm{z}_e; e\in E(\tau)\}$,
where $\|$ means vector concatenation along the edge feature axis.
The conditional branch length distribution $Q_{\bm{\psi}}(\bm{q}|\tau,\bm{z})$ in equation (\ref{eq:si-dist}) takes the form (i.e., a diagonal lognormal distribution)
\[
Q_{\bm{\psi}}(\bm{q}|\tau,\bm{z}) = \prod_{e\in E(\tau)}p^{\mathrm{Lognormal}}(q_e|\mu(e,\tau,\bm{z}), \sigma(e,\tau,\bm{z})),
\] 
where the mean and standard deviation parameters are parametrized using MLPs as follows:
\[
\mu(e,\tau,\bm{z}) = \mathrm{MLP}^{\mu}\left(\bar{\bm{h}}_e\right),\quad \sigma(e,\tau,\bm{z}) = \mathrm{MLP}^{\sigma}\left(\bar{\bm{h}}_e\right),
\]
and $\bm{\psi}$ are all the learnable parameters in this conditional branch length distribution construction.
Although the mixing distribution $Q_{\bm{\psi}}(\bm{z}|\tau)$ can also be parameterized using learnable node features of $\tau$,
here we use the simple standard Gaussian distribution for $Q_{\bm{\psi}}(\bm{z}|\tau)$ which ignores the dependency on $\tau$ for simplicity.

\subsection{Multi-sample Semi-implicit Lower Bound for VBPI-SIBranch}
Due to the semi-implicit construction of branch length distributions, the MLB  $L^K(\bm{\phi},\bm{\psi})$ in equation (\ref{eq-lower-bound}) is no longer tractable.
However, we can use a multi-sample extension of the SILB in \citet{yin2018sivi} for training. 
Letting $Q_{\bm{\phi}}(\tau)$ be the variational distribution over tree topologies, the \textit{multi-sample semi-implicit lower bound} (MSILB) is defined as 
\begin{equation}\label{eq:sivi_vbpi_lower_bound}
\begin{split}
    L^{K,J}(\bm{\phi},\bm{\psi}) = 
    \mathbb{E}_{\prod_{k=1}^K Q_{\bm{\phi},\bm{\psi}}(\tau^k,\bm{q}^k, \bm{z}^{k,0})} \mathbb{E}_{\prod_{k=1}^K Q_{\bm{\psi}}(\bm{z}^{k,1:J}|\tau^k)}
    \log\left(\frac{1}{K}\sum_{k=1}^K\frac{P(\bm{Y}|\tau^k,\bm{q}^k)P(\tau^k,\bm{q}^k)}{Q_{\bm{\phi}}(\tau^k)\frac{1}{J+1}\sum_{j=0}^JQ_{\bm{\psi}}(\bm{q}^k|\tau^k,\bm{z}^{k,j})}\right),
\end{split}
\end{equation}
where $Q_{\bm{\phi},\bm{\psi}}(\tau,\bm{q},\bm{z})=Q_{\bm{\psi}}(\bm{q}|\tau,\bm{z}) Q_{\bm{\psi}}(\bm{z}|\tau)Q_{\bm{\phi}}(\tau)$ and $Q_{\bm{\psi}}(\bm{z}^{k,1:J}|\tau^k)=\prod_{j=1}^J Q_{\bm{\psi}}(\bm{z}^{k,j}|\tau^k)$.
In fact, the above MSILB is a lower bound of the MLB $L^K(\bm{\phi},\bm{\psi})$, as proved in Theorem \ref{thm:sivi}.

\begin{theorem}\label{thm:sivi}
The MSILB $L^{K,J}(\bm{\phi},\bm{\psi})$ in equation (\ref{eq:sivi_vbpi_lower_bound}) is a lower bound of $L^K(\bm{\phi},\bm{\psi})$ in equation (\ref{eq-lower-bound}), and is an increasing function of $J$, i.e.,
$
L^{K,J}(\bm{\phi},\bm{\psi}) \leq L^{K,J+1}(\bm{\phi},\bm{\psi}) \leq L^{K}(\bm{\phi},\bm{\psi}),\; \forall J.
$
Moreover, it is asymptotically unbiased, i.e.,
$\lim_{J\to\infty}L^{K,J}(\bm{\phi},\bm{\psi})=L^K(\bm{\phi},\bm{\psi})$.
\end{theorem}

\begin{algorithm}[t]
\caption{VBPI-SIBranch with MSILB}
\label{alg:sivbpi}
\KwIn{Observed sequences $\bm{Y}\in\Omega^{N\times S}$; initialized parameters $\bm{\phi},\bm{\psi}$.}
\While{not converged}{
$\tau^1,\ldots,\tau^K \leftarrow$ independent samples from the current tree topology approximating distribution $Q_{\bm{\phi}}(\tau)$\;
\For{$k=1,\ldots,K$}{
$\bm{z}^{k,0},\ldots,\bm{z}^{k,J}\leftarrow$ independent samples form the current mixing distribution $Q_{\bm{\psi}}(\bm{z}|\tau^k)$\;
$\bm{q}^{k}\leftarrow$ a sample from the current conditional branch length distribution $Q_{\bm{\psi}}(\bm{q}|\tau^k,\bm{z}^{k,0})$\;
Calculate the conditional probabilities $Q_{\bm{\psi}}(\bm{q}^k|\tau^k,\bm{z}^{k,j})$ for $0\leq j\leq J$\;
}
$\hat{\bm{g}}\leftarrow$ the estimate of the gradient $\nabla_{\bm{\phi},\bm{\psi}}L^{K,J}(\bm{\phi},\bm{\psi})$\;
$\bm{\phi},\bm{\psi} \leftarrow$ Updated parameters using gradient estimate $\hat{\bm{g}}$.
}
\end{algorithm}

The gradient of the surrogate function (\ref{eq:sivi_vbpi_lower_bound}) w.r.t.~$\bm\phi$ and $\bm\psi$ can be estimated by the VIMCO estimator and a reparameterization trick respectively.
Therefore, the MSILB in equation \eqref{eq:sivi_vbpi_lower_bound} can be maximized the same way as in \cite{Zhang2019VariationalBP}.
We summarize the VBPI-SIBranch approach with MSILB in Algorithm \ref{alg:sivbpi}.

\subsection{Multi-sample Importance Weighted Lower Bound for VBPI-SIBranch}
The MSILB $L^{K,J}(\bm{\phi},\bm{\psi})$ for VBPI-SIBranch relies on samples $\bm{z}^{k,1:J}$ from the mixing distribution $Q_{\bm{\psi}}(\bm{z}|\tau^k)$ to estimate the marginal densities of the branch length sample $\bm{q}^{k}$, for $1\leq k\leq K$.
However, these uninformed samples may miss the high posterior domain of $Q_{\bm{\phi},\bm{\psi}}(\bm{z}|\tau^k,\bm{q}^k)$ and become less efficient in high-dimensional settings, e.g., conditional branch length distributions $Q_{\bm{\psi}}(\bm{q}^k|\tau^k,\bm{z}^{k,j})$ for $1\leq j\leq J$ can be close to zero.
Similarly to \citet{Sobolev2019IWHVI}, one may employ an auxiliary reverse model $R_{\bm{\xi}}(\bm{z}|\tau,\bm{q})$ as an importance distribution that can adapt to the high posterior domain automatically.
More precisely, we consider the following \textit{multi-sample importance weighted lower bound} (MIWLB)
\begin{equation}\label{eq:iwhvi_vbpi_lower_bound}
\begin{aligned}
L^{K,J}_{w}(\bm{\phi},\bm{\psi},\bm{\xi}) = \mathbb{E}_{\prod_{k=1}^K Q_{\bm{\phi},\bm{\psi}}(\tau^k,\bm{q}^k, \bm{z}^{k,0})}&  \mathbb{E}_{\prod_{k=1}^K R_{\bm{\xi}}(\bm{z}^{k,1:J}|\tau^k,\bm{q}^k)}\\
& \log\left(\frac{1}{K}\sum_{k=1}^K\frac{P(\bm{Y}|\tau^k,\bm{q}^k)P(\tau^k,\bm{q}^k)}{Q_{\bm{\phi}}(\tau^k)\frac{1}{J+1}\sum_{j=0}^J \frac{Q_{\bm{\psi}}(\bm{q}^k|\tau^k,\bm{z}^{k,j})Q_{\bm{\psi}}(\bm{z}^{k,j}|\tau^k)}{R_{\bm{\xi}}(\bm{z}^{k,j}|\tau^k,\bm{q}^k)}}\right),
\end{aligned}
\end{equation}
where $Q_{\bm{\phi},\bm{\psi}}(\tau,\bm{q},\bm{z})=Q_{\bm{\psi}}(\bm{q}|\tau,\bm{z}) Q_{\bm{\psi}}(\bm{z}|\tau)Q_{\bm{\phi}}(\tau)$, $R_{\bm{\xi}}(\bm{z}^{k,1:J}|\tau^k,\bm{q}^k)=\prod_{j=1}^J R_{\bm{\xi}}(\bm{z}^{k,j}|\tau^k,\bm{q}^k)$, and ``$w$'' is the abbreviation of ``weighted''.
Note that the MIWLB $L^{K,J}_{w}(\bm{\phi},\bm{\psi},\bm{\xi})$ becomes MSILB $L^{K,J}(\bm{\phi},\bm{\psi})$ if we take the reverse model $R_{\bm{\xi}}(\bm{z}|\tau,\bm{q})=Q_{\bm{\psi}}(\bm{z}|\tau)$.
Moreover, $L^{K,J}_{w}(\bm{\phi},\bm{\psi},\bm{\xi})$ is also a lower bound of the MLB $L^K(\bm{\phi},\bm{\psi})$ in equation (\ref{eq-lower-bound}), as proved in Theorem \ref{thm:iwhvi}.

\begin{theorem}\label{thm:iwhvi}
The MIWLB $L^{K,J}_w(\bm{\phi},\bm{\psi},\bm{\xi})$ in equation (\ref{eq:iwhvi_vbpi_lower_bound}) is a lower bound of the MLB $L^K(\bm{\phi},\bm{\psi})$ in equation (\ref{eq-lower-bound}), and is an increasing function of $J$, i.e.,
$
L^{K,J}_w(\bm{\phi},\bm{\psi},\bm{\xi}) \leq L^{K,J+1}_w(\bm{\phi},\bm{\psi},\bm{\xi}) \leq L^{K}(\bm{\phi},\bm{\psi}), \; \forall J,
$
for arbitrary choices of $\bm{\xi}$.
Moreover, it is asymptotically unbiased, i.e., $\lim_{J\to\infty}L^{K,J}_w(\bm{\phi},\bm{\psi},\bm{\xi})=L^K(\bm{\phi},\bm{\psi})$.
\end{theorem}

There are many choices for the reverse model $R_{\bm{\xi}}(\bm{z}|\tau,\bm{q})$, e.g., normal distributions and normalizing flows.
For simplicity, here we use a diagonal normal distribution
\[
R_{\bm{\xi}}(\bm{z}|\tau,\bm{q}) = \prod_{e\in E(\tau)} p^{\mathrm{Normal}}\left(\bm{z}_e|\bm{\mu}_R(e,\tau,\bm{q}), \bm{\sigma}_R(e,\tau,\bm{q})\right),
\]
where $\bm{\mu}_R(e,\tau,\bm{q})$ and $\bm{\sigma}_R(e,\tau,\bm{q})$ are the mean and standard deviation that are parameterized with MLPs using the edge features
\[
\bm{\mu}_R(e,\tau,\bm{q})=\mathrm{MLP}_R^{\mu}(\bm{h}_e\|q_e), \quad \bm{\sigma}_R(e,\tau,\bm{q})=\mathrm{MLP}_R^{\sigma}(\bm{h}_e\|q_e),
\]
where $\|$ means vector concatenation.
This way, the gradient of MIWLB $L^{K,J}_w(\bm{\phi},\bm{\psi},\bm{\xi})$ w.r.t. $\bm{\xi}$ can be estimated by the reparameterization trick.
We summarize the VBPI-SIBranch approach with MIWLB in Algorithm \ref{alg:iwvbpi}.
\begin{algorithm}[t]
\caption{VBPI-SIBranch with MIWLB}
\label{alg:iwvbpi}
\KwIn{Observed sequences $\bm{Y}\in\Omega^{N\times S}$; initialized parameters $\bm{\phi},\bm{\psi},\bm{\xi}$.}
\While{not converged}{
$\tau^1,\ldots,\tau^K \leftarrow$ independent samples from the current tree topology approximating distribution $Q_{\bm{\phi}}(\tau)$\;
\For{$k=1,\ldots,K$}{
$\bm{z}^{k,0} \leftarrow$ a sample form the current mixing distribution $Q_{\bm{\psi}}(\bm{z}|\tau^k)$\;
$\bm{q}^{k}\leftarrow$ a sample from the current conditional branch length distribution $Q_{\bm{\psi}}(\bm{q}|\tau^k,\bm{z}^{k,0})$\;
$\bm{z}^{k,1},\ldots,\bm{z}^{k,J} \leftarrow$ independent samples from the reverse distribution $R_{\bm{\xi}}(\bm{z}|\tau^k,\bm{q}^k)$\;
Calculate $Q_{\bm{\psi}}(\bm{q}^k|\tau^k,\bm{z}^{k,j}), Q_{\bm{\psi}}(\bm{z}^{k,j}|\tau^k)$ and $R_{\bm{\xi}}(\bm{z}^{k,j}|\tau^k,\bm{q}^k)$ for $0\leq j\leq J$\;
}
$\hat{\bm{g}}\leftarrow$ the estimate of the gradient $\nabla_{\bm{\phi},\bm{\psi},\bm{\xi}}L^{K,J}_w(\bm{\phi},\bm{\psi},\bm{\xi})$\;
$\bm{\phi},\bm{\psi},\bm{\xi} \leftarrow$ Updated parameters using gradient estimate $\hat{\bm{g}}$.
}
\end{algorithm}
\section{Experiments}\label{sec:experiments}

In this section, we test the effectiveness of VBPI-SIBranch on two common tasks for Bayesian phylogenetic inference: marginal likelihood estimation and posterior approximation.
Our code is available at \href{https://github.com/tyuxie/VBPI-SIBranch}{\texttt{https://github.com/tyuxie/VBPI-SIBranch}}.

\subsection{Experimental Setup}
\paragraph{Targets} 
The experiments are performed on eight benchmark data sets which we will call DS1-8. 
These data sets consist of nucleotide sequences from 27 to 64 eukaryote species with 378 to 2520 sites and 
are commonly used to benchmark the Bayesian phylogenetic inference task in previous works \citep{Zhang2019VariationalBP, Zhang2023learnable, mimori2023geophy, xie2023artree, Zhou2023PhyloGFN}.
We assume a uniform prior on the tree topologies, an i.i.d. exponential prior $\mathrm{Exp(10)}$ on branch lengths, and the simple Jukes \& Cantor (JC) substitution model \citep{jukes1969evolution}.

\paragraph{Variational Family} 
We use the same tree topology variational distribution $Q_{\bm{\phi}}(\tau)$, i.e., the simplest SBNs, for all branch length variational distributions.
The conditional probability supports for SBNs are gathered from 10 replicates of 10000 maximum likelihood bootstrap trees using UFBoot \citep{Minh2013UltrafastAF}, following \citet{Zhang2019VariationalBP}.
For the branch lengths, we compare our semi-implicit variational approximation to two baselines: VBPI \citep{Zhang2023learnable} and VBPI-NF \citep{Zhang2020ImprovedVB}.  
To obtain the learnable topological node features, both VBPI-SIBranch and VBPI use the same architecture for GNNs, which contain $L=2$ rounds of message passing steps with the aggregation function and update function following the edge convolution operator \citep{Wang2018DynamicGC}.
On all data sets, we set the dimension of learnable topological node features to 100 and the dimension of hidden variables to 50.
All the activation functions in MLPs are exponential linear units (ELUs) \citep{ELU}.
For VBPI-NF, we use the best RealNVP \citep{dinh2016realnvp} model with 10 layers to model the branch lengths, following \citet{Zhang2020ImprovedVB}.

\paragraph{Optimization} 
We set the number of particles $K=10$ for all the MLB, MSILB, and MIWLB. 
For both MSILB and MIWLB, we set the number of extra samples to be $J=50$.  
To accommodate the multimodality of phylogenetic posterior, we target the annealed phylogenetic posterior at the $i$-th iteration:
$P(\bm{Y},\tau,\bm{q}; \lambda_i) = P(\bm{Y}|\tau,\bm{q})^{\lambda_i}P(\tau,\bm{q})$,
where the annealing schedule $\lambda_i = \min(1, 0.001+i/100000)$ goes from $0.001$ to $1$ after 100000 iterations. 
The gradient estimates for the tree topology parameters are obtained by the VIMCO estimator \citep{Mnih2016VariationalIF}, and those for the branch length parameters and reverse model parameters are obtained by the reparameterization trick \citep{VAE}.
All these models are implemented in PyTorch \citep{Paszke2019PyTorchAI} and trained with the Adam optimizer \citep{ADAM}.
The learning rate is 0.001 for the tree topology model, 0.001 for the branch length model in VBPI and VBPI-SIBranch, and 0.0001 for the branch length model in VBPI-NF.
All results are collected after 400000 parameter updates.

\subsection{Marginal Likelihood Estimation}
\begin{figure}[t]
\centering
\includegraphics[width=\linewidth]{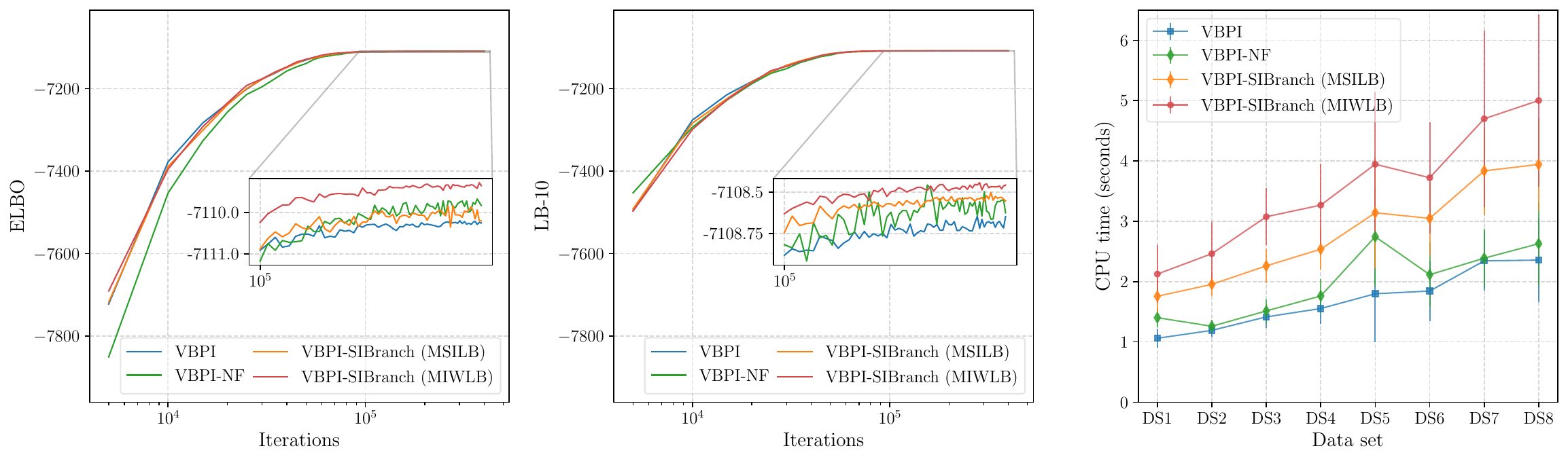}
\caption{Visualization of the training processes of different methods for VBPI.
\textbf{Left}: evidence lower bound (ELBO, estimated using $J=1000$ extra samples) as a function of iterations on DS1.
\textbf{Middle}: 10-sample lower bound (LB-10, estimated using $J=1000$ extra samples) as a function of iterations on DS1. 
\textbf{Right}: Time cost per 10 training iterations of different methods on a single core of Intel Xeon Platinum 9242 processor.
The results are averaged over 100 runs with the standard deviation as the error bar.
}
\label{fig:vbpi-training}
\end{figure}

\begin{table}[t]
\centering
\renewcommand\arraystretch{1.25}
\setlength{\tabcolsep}{0.1cm}{}
\caption{Evidence lower bound (ELBO), 10-sample lower bound (LB-10), and marginal likelihood (ML) estimates of different methods across 8 benchmark data sets. 
The MSILB$^\ast$ refers to the MIWLB estimates of the variational approximation in VBPI-SIBranch (MSILB).
The ML estimates are obtained via importance sampling using 1000 samples.
For ELBO, LB-10, and ML, the results are averaged over 100, 100, and 1000 independent runs respectively with standard deviation in the brackets.
Results of stepping-stone (SS) are from \citet{Zhang2019VariationalBP}.
}
\label{tab:vbpi-lower-bound}
\vskip0.5em
\begin{footnotesize}
\resizebox{\linewidth}{!}{
\hspace{-0.2in}
\begin{tabular}{cccccccccc}
\cmidrule[1pt]{2-10}
 & Data set & DS1 & DS2 & DS3 & DS4 & DS5 & DS6 & DS7 & DS8 \\
 & \# Taxa & 27 &  29 &  36 & 41 & 50 &  50 & 59   &  64  \\
 & \# Sites & 1949  & 2520 & 1812  & 1137  & 378   & 1133  & 1824  &  1008  \\
\cmidrule[0.5pt]{2-10}
\multirow{5}{*}{\rot{ELBO}} & VBPI-SIBranch (MSILB) &-7110.00(0.30) & -26368.66(0.09) & -33736.07(0.07) & -13331.60(0.32)&-8217.31(0.20)&-6728.25(0.44)&-37334.41(0.34)&-8654.55(0.32)\\
& VBPI-SIBranch (MSILB$^\ast$) & -7109.99(0.28) &-26368.66(0.09) & -33736.06(0.07)&-13331.59(0.29)&-8217.29(0.21)&-6728.21(0.44)&-37334.39(0.34)&-8654.49(0.33)\\
&VBPI-SIBranch (MIWLB) &\textbf{-7109.34(0.13)}&-26368.56(0.09)&-33735.93(0.06)&\textbf{-13330.81(0.08)}&\textbf{-8215.95(0.09)}&\textbf{-6725.05(0.07)}&\textbf{-37333.22(0.09)}&\textbf{-8651.49(0.09)}\\
&VBPI &-7110.26(0.10) &-26368.84(0.09)& -33736.25(0.08)& -13331.80(0.10) &-8217.80(0.12)& -6728.57(0.16)& -37334.84(0.14)& -8655.01(0.14)\\
&VBPI-NF&-7109.83(0.10)&\textbf{-26368.44(0.19)}&\textbf{-33735.73(0.10)}&-13331.36(0.09)&-8217.59(0.10)&-6728.04(0.14)&-37333.85(0.09)&-8654.10(0.12)\\
 \cmidrule[0.5pt]{2-10}
\multirow{5}{*}{\rot{LB-10}} & VBPI-SIBranch (MSILB) & -7108.53(0.02) & -26367.82(0.02) & -33735.22(0.02) & -13330.12(0.02)&-8215.03(0.03)&-6724.81(0.03)&-37332.30(0.03)&-8651.26(0.04)\\
& VBPI-SIBranch (MSILB$^\ast$) & -7108.53(0.02) & -26367.82(0.01) &-33735.23(0.02) &-13330.11(0.02)&-8215.01(0.03)&-6724.77(0.03)&-37332.29(0.03)&-8651.22(0.04)\\
& VBPI-SIBranch (MIWLB) & \textbf{-7108.46(0.01)} & -26367.80(0.01)&-33735.20(0.01)& \textbf{-13330.02(0.01)}&\textbf{-8214.70(0.02)}&\textbf{-6724.26(0.01)}&\textbf{-37332.11(0.02)}&\textbf{-8650.49(0.02)}\\
   &VBPI &-7108.69(0.02) & -26367.87(0.02) & -33735.26(0.02) & -13330.29(0.02) & -8215.42(0.04) & -6725.33(0.04) & -37332.58(0.03) & -8651.78(0.04) \\
&VBPI-NF&-7108.58(0.02)&\textbf{-26367.75(0.01)}&\textbf{-33735.15(0.01)}&-13330.15(0.02)&-8215.30(0.03)&-6725.18(0.04)&-37332.29(0.03)&-8651.43(0.04)\\
  \cmidrule[0.5pt]{2-10}
\multirow{6}{*}{\rot{ML}} & VBPI-SIBranch (MSILB) & -7108.39(0.07) & -26367.71(0.06)& -33735.09(0.07) & -13329.91(0.10)&-8214.48(0.27)&-6724.21(0.25)&-37331.91(0.16)&-8650.44(0.35)\\
& VBPI-SIBranch (MSILB$^\ast$) & -7108.39(0.06) & -26367.71(0.05)& -33735.09(0.07) & -13329.91(0.09)&-8214.47(0.28)&-6724.20(0.23)&-37331.91(0.15)&-8650.43(0.33)\\
&VBPI-SIBranch (MIWLB) & \textbf{-7108.39(0.04)}&-26367.71(0.05)&-33735.09(0.07)&\textbf{-13329.91(0.06)}&\textbf{-8214.43(0.19)}&\textbf{-6724.16(0.06)}&\textbf{-37331.90(0.09)}&\textbf{-8650.33(0.11)}\\
&VBPI&-7108.41(0.15)&-26367.71(0.08)&-33735.09(0.09)&-13329.94(0.20)& -8214.62(0.40) & -6724.37(0.43)& -37331.97(0.28) & -8650.64(0.50)\\
&VBPI-NF&-7108.39(0.17)&\textbf{-26367.71(0.03)}&\textbf{-33735.09(0.05)}&-13329.92(0.15)&-8214.59(0.45)&-6724.33(0.42)&-37331.93(0.18)&-8650.55(0.39)\\
&SS& -7108.42(0.18)& -26367.57(0.48)& -33735.44(0.50)& -13330.06(0.54) &-8214.51(0.28)& -6724.07(0.86) &-37332.76(2.42) &-8649.88(1.75)\\
\cmidrule[1pt]{2-10}
\end{tabular}
}
\end{footnotesize}
\end{table}

We first investigate the performances of different methods for estimating the marginal likelihood and its lower bounds.
Figure \ref{fig:vbpi-training} depicts the training processes and the time costs for VBPI on DS1.
We see that the ELBO and the 10-sample lower bound (LB-10) as functions of iterations for VBPI-SIBranch align with those for VBPI and VBPI-NF.
Moreover, VBPI-SIBranch with MIWLB finally achieves the best lower bounds compared to the other three methods. 
In the right plot of Figure \ref{fig:vbpi-training}, we find that VBPI-SIBranch requires comparable time in training although multiple extra samples  ($J=50$) are needed, due to the efficient vectorized implementation.
Table \ref{tab:vbpi-lower-bound} shows the ELBO, LB-10, and marginal likelihood (ML) estimates of different methods on DS1-8.
It is worth noting that the comparison between VBPI-SIBranch (MSILB) and VBPI-SIBranch (MIWLB) might be unfair since they use different importance distributions for evaluation.
Therefore, we train a reverse model for the variational approximation in VBPI-SIBranch (MSILB) and calculate the lower bound estimates using MIWLB. Results in this setting are reported in VBPI-SIBranch (MSILB$^\ast$).
We see that VBPI-SIBranch consistently outperforms the VBPI baseline in terms of lower bounds and marginal likelihood estimates, indicating the effectiveness of semi-implicit branch length distributions.
Moreover, the superior performance of VBPI-SIBranch (MIWLB) over VBPI-SIBranch (MSILB) and VBPI-SIBranch (MSILB$^\ast$) suggests that employing a learnable importance distribution can be beneficial for the training of VBPI-SIBranch.

\subsection{Posterior Approximation}

\begin{figure}[t]
\centering
\includegraphics[width=\linewidth]{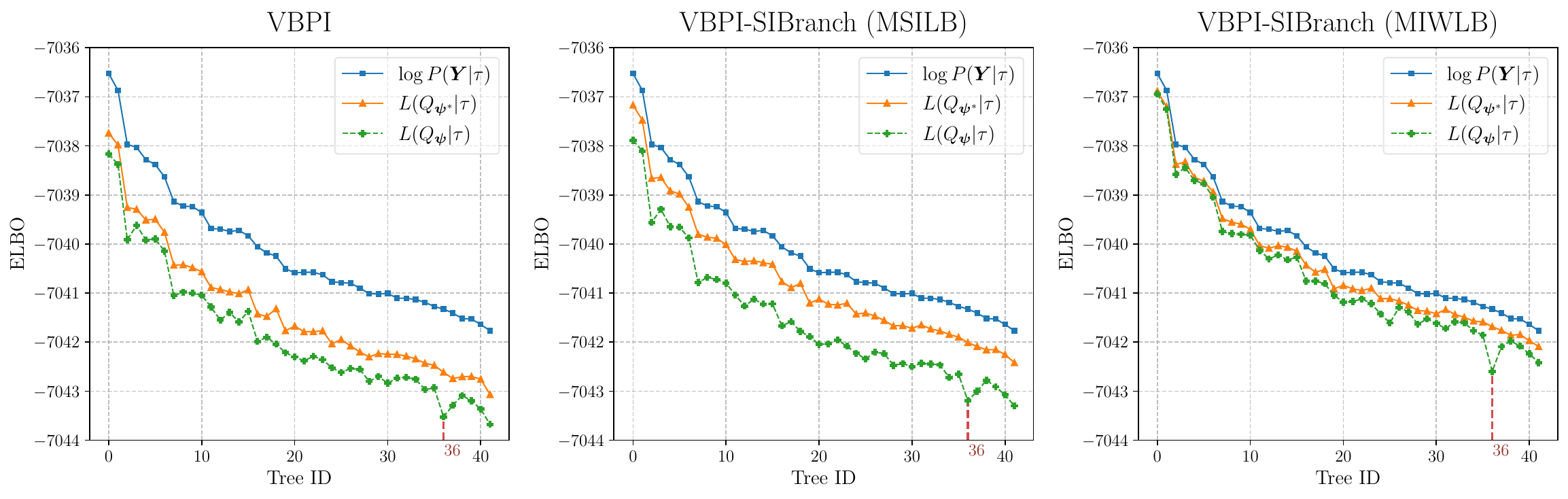}
\caption{Inference gaps on tree topologies in the 95\% credible set of DS1. 
The $L(Q_{\bm{\psi}}|\tau)$ refers to the ELBO of the variational approximation, and the $L(Q_{\bm{\psi}^\ast}|\tau)$ refers to the best ELBO that can be achieved by the corresponding variational family.
All lower bounds were computed by averaging over 10000 Monte Carlo samples. The ground truth marginal log-likelihood $\log P(\bm{Y}|\tau)$ is estimated using the generalized stepping-stone (GSS) algorithm \citep{Fan2010GSS}. 
}
\label{fig:inference-gap}
\end{figure}

\begin{table}[t]
    \centering
    \caption{Inference gaps on tree topologies in the 95\% credible set of DS1. The Avg. column refers to the average gaps over all tree topologies in the credible set. Results of VBPI-NF are from \citet{Zhang2020ImprovedVB}.}
\label{tab:my_label}
\resizebox{\linewidth}{!}{
    \begin{tabular}{ccccccccc}
\toprule
\multirow{2}{*}{Gap} &\multicolumn{2}{c}{VBPI}&\multicolumn{2}{c}{VBPI-NF}&\multicolumn{2}{c}{VBPI-SIBranch (MISLB)}&\multicolumn{2}{c}{VBPI-SIBranch (MIWLB)} \\
\cmidrule(l){2-3}\cmidrule(l){4-5}\cmidrule(l){6-7}\cmidrule(l){8-9}
 &  Avg.& Tree 36& Avg.&Tree 36& Avg.&Tree 36&Avg.&Tree 36\\
\midrule
Approximation&1.22&1.29&0.40&0.43&0.64&0.68&\textbf{0.34}&\textbf{0.36}\\
Amortization&0.51&\textbf{0.91}&0.93&1.83&0.80&1.19&\textbf{0.22}&\textbf{0.91}\\
Inference&1.73&2.20&1.33&2.26&1.44&1.87&\textbf{0.56}&\textbf{1.27}\\
\bottomrule
\end{tabular}
}
\end{table}
\paragraph{Inference Gaps on Individual Trees} To better understand the effect of semi-implicit branch length distributions for the overall improvement on variational approximation accuracy, we further evaluate the performance of different methods on individual trees in the 95\% credible set of DS1.
For a fixed tree topology $\tau$, we define the ELBO $L(Q_{\bm{\psi}}|\tau)$ of a variational approximation $Q_{\bm{\psi}}(\bm{q}|\tau)$ and the best ELBO that can be achieved by the corresponding variational family $\mathcal{Q}$ as
\[
L(Q_{\bm{\psi}}|\tau)
= \mathbb{E}_{Q_{\bm{\psi}}(\bm{q}|\tau)}\log\left(\frac{P(\bm{Y}|\tau,\bm{q})P(\bm{q})}{Q_{\bm{\psi}}(\bm{q}|\tau)}\right),\
L(Q_{\bm{\psi}^\ast}|\tau)
= \max_{Q_{\bm{\psi}}\in\mathcal{Q}} L(Q_{\bm{\psi}}|\tau).
\]
If $Q_{\bm{\psi}}(\bm{q}|\tau)$ is semi-implicit as in equation (\ref{eq:si-marginal-dist}), one may imitate MSILB and MIWLB to estimate $L(Q_{\bm{\psi}}|\tau)$ and $L(Q_{\bm{\psi}^\ast}|\tau)$, i.e.
\[
L(Q_{\bm{\psi}}|\tau) \approx L^J(Q_{\bm{\psi}}|\tau)
= \mathbb{E}_{Q_{\bm{\psi}}(\bm{q}, \bm{z}^{0}|\tau)Q_{\bm{\psi}}(\bm{z}^{1:J}|\tau)}
    \log\left(\frac{P(\bm{Y}|\tau,\bm{q})P(\bm{q})}{\frac{1}{J+1}\sum_{j=0}^JQ_{\bm{\psi}}(\bm{q}|\tau,\bm{z}^{j})}\right),\
L(Q_{\bm{\psi}^\ast}|\tau)
\approx \max_{Q_{\bm{\psi}}\in\mathcal{Q}} L^J(Q_{\bm{\psi}}|\tau),
\]
in the MSILB setting, or 
\[
\begin{array}{rcl}
L(Q_{\bm{\psi}}|\tau) &\approx & L^J_w(Q_{\bm{\psi}}, R_{\bm{\xi}}|\tau)
=  \mathbb{E}_{Q_{\bm{\psi}}(\bm{q}, \bm{z}^{0}|\tau)R_{\bm{\xi}}(\bm{z}^{1:J}|\tau,\bm{q})}
 \log\left(\frac{P(\bm{Y}|\tau,\bm{q})P(\bm{q})}{\frac{1}{J+1}\sum_{j=0}^J \frac{Q_{\bm{\psi}}(\bm{q}|\tau,\bm{z}^{j})Q_{\bm{\psi}}(\bm{z}^{j}|\tau)}{R_{\bm{\xi}}(\bm{z}^{i,j}|\tau^i,\bm{q}^i)}}\right),\\
L(Q_{\bm{\psi}^\ast}|\tau)
&\approx &\max_{Q_{\bm{\psi}}\in\mathcal{Q}, R_{\bm{\xi}}\in\mathcal{R}} L^J_w(Q_{\bm{\psi}}, R_{\bm{\xi}}|\tau),
\end{array}
\]
in the MIWLB setting. 
To compute the best ELBO $L(Q_{\bm{\psi}^\ast}|\tau)$, we take $J=50$ for training and $J=1000$ for evaluation in practice.
For a fixed tree topology $\tau$, the inference gap of each variational family is defined as the difference between the marginal log-likelihood $\log P(\bm{Y}|\tau)$ and the ELBO $L(Q_{\bm{\psi}}|\tau)$, which can be decomposed as
\[
\log P(\bm{Y}|\tau)-L(Q_{\bm{\psi}}|\tau) = \left[\log P(\bm{Y}|\tau)-L(Q_{\bm{\psi}}^\ast|\tau)\right] + \left[L(Q_{\bm{\psi}}^\ast|\tau)-L(Q_{\bm{\psi}}|\tau)\right],
\]
i.e., the sum of approximation and amortization gaps \citep{Cremer2018InferenceSI, Zhang2020ImprovedVB}.

Figure \ref{fig:inference-gap} shows the decomposition of the inference gap of different variational families on DS1.
In the left plot of Figure \ref{fig:inference-gap}, the large approximation gap indicates that the diagonal lognormal distribution in VBPI is too restricted to fit the true branch length distribution. 
In contrast, the semi-implicit branch length distribution in VBPI-SIBranch performs much better, as indicated by the considerably smaller approximation gaps in the middle and right plots.
Moreover, compared to VBPI-SIBranch with MSILB, VBPI-SIBranch with MIWLB significantly reduces the approximation gap and generalizes better to the tree topology space by employing a learnable importance distribution, as evidenced by the reduction of the amortization gap.

\paragraph{Branch Length Approximation}
To examine the approximation accuracy of the learned branch length model $Q_{\bm{\psi}}(\bm{q}|\tau)$ to the ground truth $P(\bm{q}|\tau,\bm{Y})$ more directly, we compare their empirical density functions estimated from branch length samples.
This also excludes the effects from the importance distribution in the lower bound comparison.
The total variation (TV) distance between two distributions 
with probability density function $P_1(\bm{x})$ and $P_2(\bm{x})$ ($\bm{x}\in\mathbb{R}^d$) is defined as 
\[
D_{\mathrm{TV}}(P_1\| P_2) = \frac{1}{2}\int_{\mathbb{R}^d}|P_1(\bm{x})-P_2(\bm{x})|\mathrm{d}\bm{x}.
\]
The left plot of Figure \ref{fig:tvd} shows the TV distance between the learned branch length variational distribution $Q_{\bm{\psi}}(\bm{q}|\tau)$ and the ground truth $P(\bm{q}|\tau,\bm{Y})$. 
We find that VBPI-SIBranch indeed provides a better approximation to the ground truth branch lengths than VBPI. 
Also, the variational approximation on tree 36 still has a relatively large error, which coincides with the observation in Figure \ref{fig:inference-gap}.
In fact, this relatively large approximation error of VBPI-SIBranch (MIWLB) on tree 36 is identified to be the result of the poor fitting on branch 35 (the middle right plot in Figure \ref{fig:tree37-selected}), and VBPI-SIBranch reaches better or comparable approximations on other branches.

\paragraph{Evaluation of Importance Weighting}
In the previous discussions, the importance weighting scheme as well as the learnable importance distribution employed in the MIWLB proved to be beneficial to the optimization of VBPI-SIBranch.
We now inspect the effect of important weighting more specifically.
In MIWLB, when using $R_{\bm{\xi}}(\bm{z}|\tau,\bm{q})$ as an importance distribution to estimate $Q_{\bm{\psi}}(\bm{q}|\tau) = \int Q_{\bm{\psi}}(\bm{q}|\tau,\bm{z})Q_{\bm{\psi}}(\bm{z}|\tau)\mathrm{d}\bm{z}$, the effective sample size (ESS) is defined as 
\[
\mathrm{ESS} = \mathbb{E}_{Q_{\bm{\psi}}(\bm{q}|\tau)}\mathbb{E}_{R_{\bm{\xi}}(\bm{z}^{1:J}|\tau,\bm{q})}\frac{1}{\sum_{j=1}^J w_j^2}, \quad w_j = \frac{
\frac{Q_{\bm{\psi}}(\bm{q}|\tau,\bm{z}^j)Q_{\bm{\psi}}(\bm{z}^j|\tau)}{R_{\bm{\xi}}(\bm{z}^j|\tau,\bm{q})}
}{
\sum_{i=1}^J\frac{ Q_{\bm{\psi}}(\bm{q}|\tau,\bm{z}^i)Q_{\bm{\psi}}(\bm{z}^i|\tau)}{R_{\bm{\xi}}(\bm{z}^i|\tau,\bm{q})}
}.
\]
ESS as a criterion is also suitable for MSILB by letting $R_{\bm{\xi}}(\bm{z}|\tau,\bm{q})=Q_{\bm{\psi}}(\bm{z}|\tau)$.
From the right plot in Figure \ref{fig:tvd}, we see that the ESS of MIWLB consistently outperforms that of MSILB, implying that the reverse model $R_{\bm{\xi}}(\bm{z}|\tau,\bm{q})$ in MIWLB indeed provides a better importance distribution.

\begin{figure}[t]
\centering
\includegraphics[width=\linewidth]{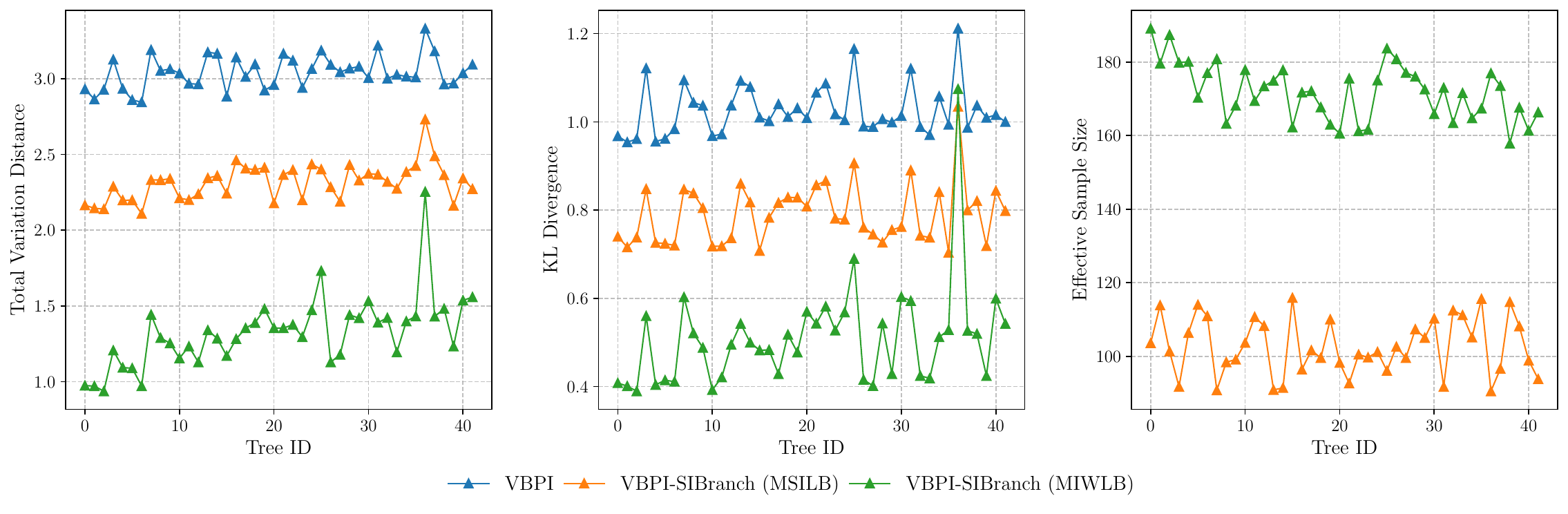}
\caption{Branch length approximation accuracy of different methods for VBPI on DS1. \textbf{Left/Middle}: The TV distance and KL divergence between the branch length variational distribution and the ground truth on individual tree topologies. \textbf{Right}: the effective sample size of the importance sampling estimation of $Q_{\bm{\psi}}(\bm{q}|\tau)$ in VBPI-SIBranch. To simplify computation, the TV distance and KL divergence are defined as $\sum_{e\in E(\tau)}D_{\mathrm{TV}}(Q_{\bm{\psi}}(q_e|\tau)\|P(q_e|\tau,\bm{Y}))$ and $\sum_{e\in E(\tau)}D_{\mathrm{KL}}(Q_{\bm{\psi}}(q_e|\tau)\|P(q_e|\tau,\bm{Y}))$, respectively, where one million samples are drawn from each distribution. The ground truth samples are gathered from a long MrBayes run with 4 chains for one billion iterations and sampled every 100 iterations.
}
\label{fig:tvd}
\end{figure}

\begin{figure}[t]
\centering
\includegraphics[width=\linewidth]{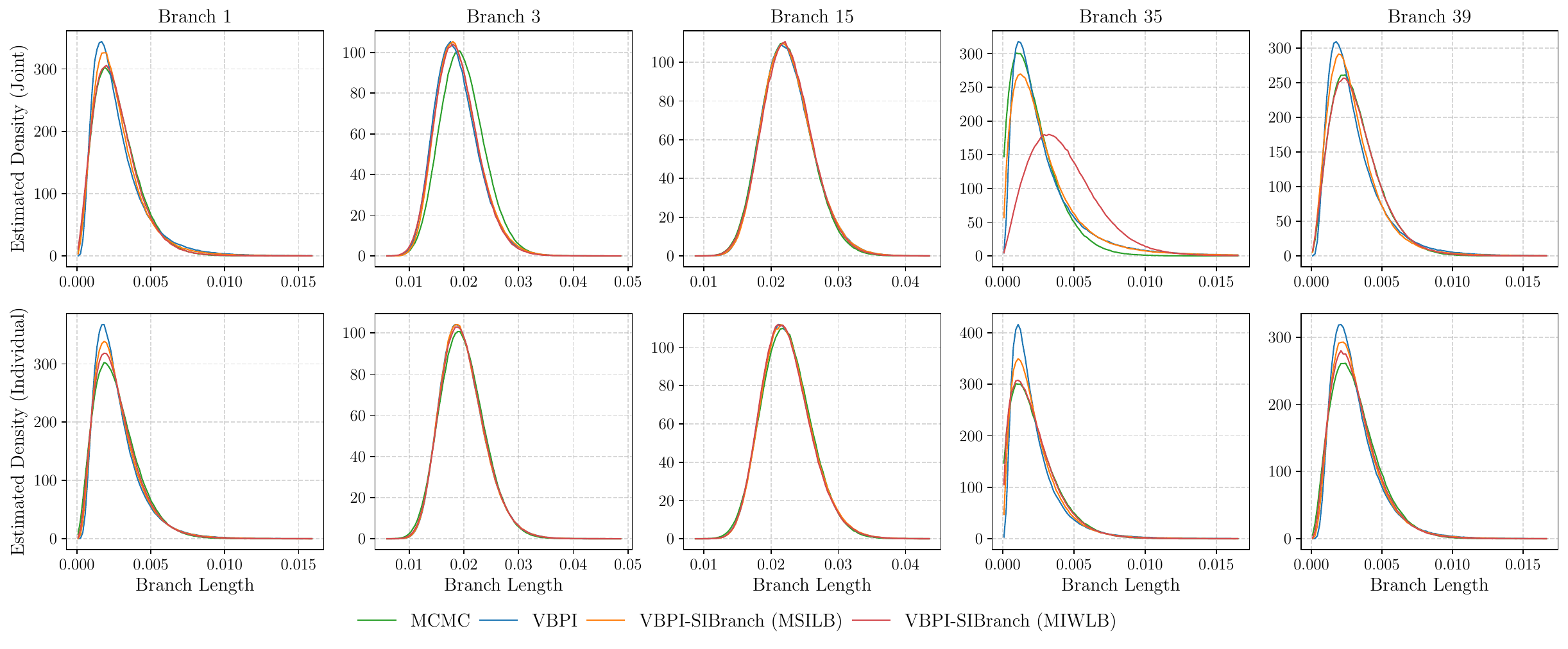}
\caption{Selected marginal branch length variational distributions obtained by different methods on tree 36 of DS1.
For each method, we estimated the probability density function with one million samples.
}
\label{fig:tree37-selected}
\end{figure}

\begin{figure}[t]
    \centering
    \includegraphics[width=0.7\linewidth]{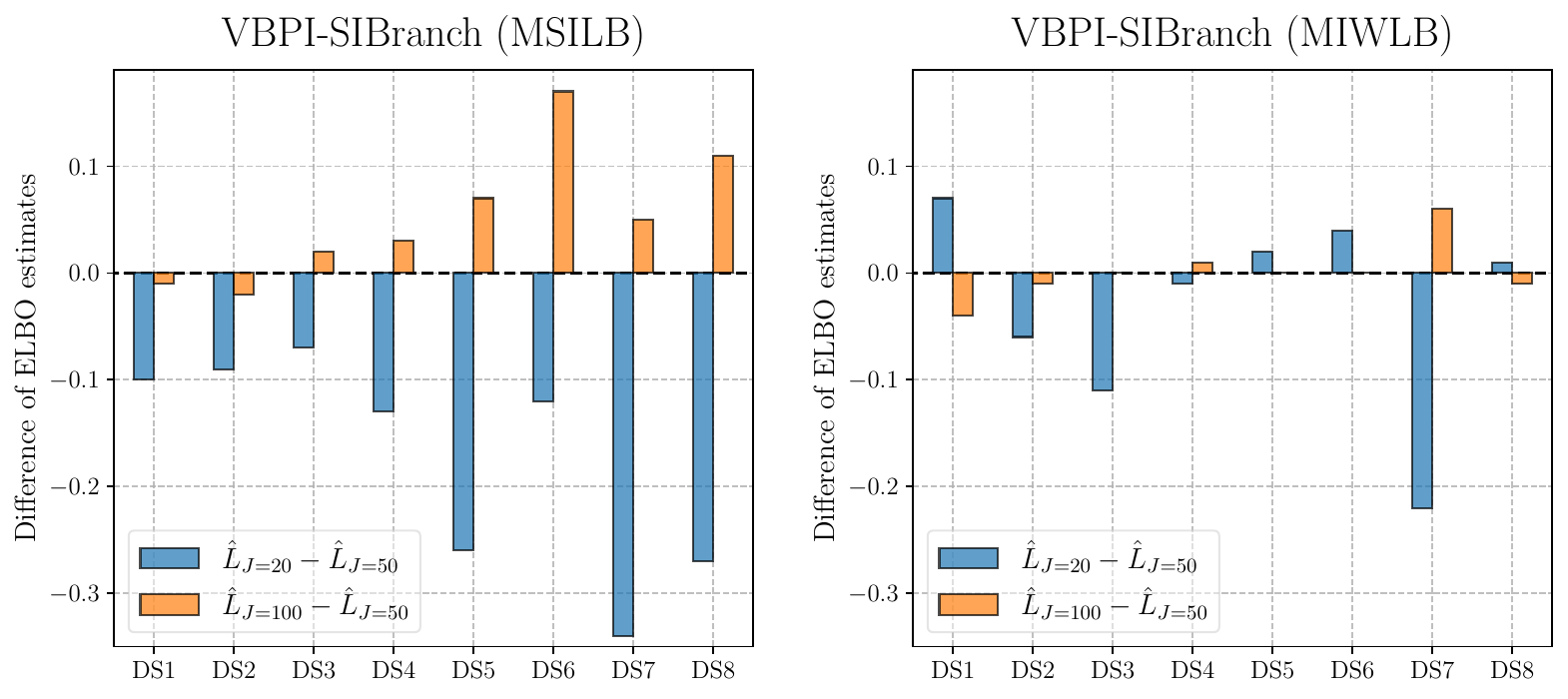}
    \caption{Ablation study about the number of extra samples $J$ in VBPI-SIBranch.
    For each method, we train the models using $K=10$ and different $J=20,50,100$, and estimate the ELBOs of the variational approximations for different training objectives ($\hat{L}_{J=20},\hat{L}_{J=50},\hat{L}_{J=100}$) with $K=1$ and $J=1000$.
    }
    \label{fig:ablation}
\end{figure}

\subsection{Ablation Studies}
Finally, we explore the effect of different numbers of extra samples $J$ on the performance of VBPI-SIBranch (Figure \ref{fig:ablation}).
We see that the ELBO estimates of VBPI-SIBranch (MSILB) get significantly better as the number of extra samples increases, while those of VBPI-SIBranch (MIWLB) exhibit randomness across different numbers of extra samples.
This implies that more extra samples are beneficial to the training of VBPI-SIBranch and MIWLB is less sensitive to the choice of the number of extra samples.

\section{Conclusion}\label{sec:conclusion}
This work presented VBPI-SIBranch, which incorporated a semi-implicit branch length model in the variational family of phylogenetic trees for VBPI.
We gave a concrete example of semi-implicit branch length distribution construction with graph neural networks.
Two surrogates of the multi-sample lower bound, i.e., multi-sample semi-implicit lower bound (MSILB) and multi-sample importance weighted lower bound (MIWLB), as training objectives were derived and their statistical properties were discussed.
Experiments on benchmark data sets demonstrated that VBPI-SIBranch achieves comparable or better results regarding marginal likelihood estimation and branch length approximation.
This work also showed the great potential of the variational inference for phylogenetic inference, aligned with some latest efforts in this domain \citep{Zhang2023learnable,xie2023artree,kviman2023vbpimixture}, and demonstrated the power of deep learning methods \citep{Zhang2023learnable} for representing phylogenetic trees.
Designing more flexible and scalable variational families for tree topologies and branch lengths based on powerful tree embeddings can be an important future direction in the field of variational phylogenetic inference.

\paragraph{Limitations}
Throughout this paper, the mixing distribution $Q_{\bm{\psi}}(\bm{z}|\tau)$ is set to a standard Gaussian distribution which ignores the dependency on $\tau$.
Designing $Q_{\bm{\psi}}(\bm{z}|\tau)$ with the information of $\tau$, e.g., learnable node features, would be an interesting future direction.

\section*{Acknowledgments}

The research of Cheng Zhang was supported in part by National Natural Science Foundation of China (grant no. 12201014 and grant no. 12292983), the National Engineering Laboratory for Big Data Analysis and Applications, the Key Laboratory of Mathematics and Its Applications (LMAM), and the Fundamental Research Funds for the Central Universities, Peking University. This research was partially supported through US National Institutes of Health grant R01 AI162611.

\bibliographystyle{nameyear}
\bibliography{main}

\newpage
\appendix
\small
\section{Details of Subsplit Bayesian Networks}\label{details-sbns}

\begin{figure}
\centering
\input{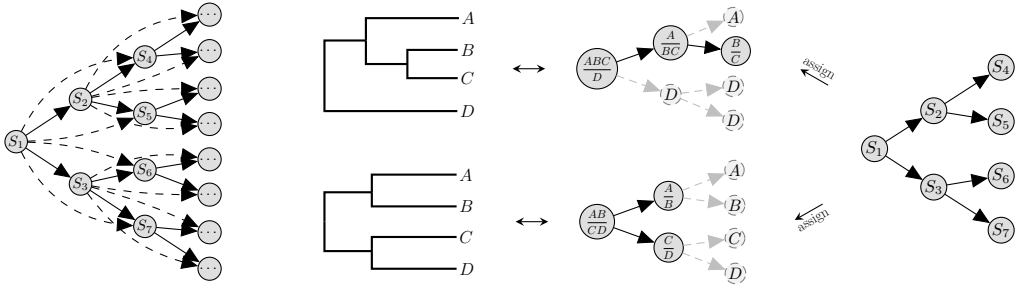}
\captionsetup{font={footnotesize}} 
\caption{Subsplit Bayesian networks and a simple example for a leaf set of 4 taxa (denoted by $A,B,C,D$ respectively).
{\bf Left:} General subsplit Bayesian networks. The solid full and complete binary tree network is $B^\ast_{\mathcal{X}}$.
The dashed arrows represent the additional dependence for more expressiveness.
{\bf Middle Left:} Examples of (rooted) phylogenetic trees that are hypothesized to model the evolutionary history of the taxa.
{\bf Middle Right:} The corresponding subsplit assignments for the trees.
For ease of illustration, subsplit $(Y,Z)$ is represented as $\frac{Y}{Z}$ in the graph.
{\bf Right:} The SBN for this example, which is $\mathcal{B}_\mathcal{X}^\ast$ in this case.
This figure is from \citet{Zhang2018GeneralizingTP}.}\label{fig:sbn}
\end{figure}

One recent and expressive graphical model that provides a flexible family of distributions over tree topologies is the subsplit Bayesian network, as proposed by \citet{Zhang2018GeneralizingTP}. 
Let $\mathcal{X}$ be the set of $N$ labeled leaf nodes.
A non-empty set $C$ of $\mathcal{X}$ is referred to as a \textit{clade} and the set of all clades of $\mathcal{X}$, denoted by $\mathcal{C}(\mathcal{X})$, is a totally ordered set with a partial order $\succ$ (e.g., lexicographical order) defined on it.
An ordered pair of clades $(W,Z)$ is called a \textit{subsplit} of a clade $C$ if it is a bipartition of $C$, i.e., $W\succ Z$, $W\cap Z=\emptyset$, and $W\cup Z=C$. 

\begin{definition}[Subsplit Bayesian Network]
A subsplit Bayesian network (SBN) $\mathcal{B}_{\mathcal{X}}$ on a leaf node set $\mathcal{X}$ of size $N$ is defined as a Bayesian network whose nodes take on subsplit or singleton clade values of $\mathcal{X}$ and has the following properties: (a) The root node of $\mathcal{B}_{\mathcal{X}}$ takes on subsplits of the entire labeled leaf node set $\mathcal{X}$; (b) $\mathcal{B}_{\mathcal{X}}$ contains a full and complete binary tree network $B^\ast_{\mathcal{X}}$ as a subnetwork; (c) The depth of $B_{\mathcal{X}}$ is $N-1$, with the root counted as depth $1$.
\end{definition}

Due to the binary structure of $B^\ast_{\mathcal{X}}$, the nodes in SBNs can be indexed by denoting the root node with $S_1$ and two children of $S_i$ with $S_{2i}$ and $S_{2i+1}$ recursively where $S_i$ is an internal node (see the left plot in Figure \ref{fig:sbn}). For any rooted tree topology, by assigning the corresponding subsplits or singleton clades values $\{S_i=s_i\}_{i\geq 1}$ to its nodes, one can uniquely map it into an SBN node assignment (see the middle and right plots in Figure \ref{fig:sbn}). 

As Bayesian networks, the SBN-based probability of a rooted tree topology $\tau$ takes the following form
\begin{equation}\label{sbn-prob}
p_{\mathrm{sbn}}(T=\tau) = p(S_1=s_1)\prod_{i>1}p(S_i=s_i|S_{\pi_i} = s_{\pi_i}),
\end{equation}
where $\pi_i$ is the index set of the parents of node $i$.
For unrooted tree topologies, we can also define their SBN-based probabilities by viewing them as rooted tree topologies with unobserved roots and integrating the positions of the root node
as follows:
\begin{equation}\label{sbn-prob-unrooted}
p_{\mathrm{sbn}}(T^{\mathrm{u}}=\tau) =\sum_{e\in E(\tau)}p_{\mathrm{sbn}}(\tau^{e}) 
\end{equation}
where $\tau^e$ is the resulting rooted tree topology when the rooting position is on edge $e$.

In practice, SBNs are parameterized according to the \textit{conditional probability sharing} principle where the conditional probability for parent-child subsplit pairs are shared across the SBN network, regardless of their locations. The set of all conditional probabilities are called conditional probability tables (CPTs).
Parameterizing SBNs, therefore, often requires finding an appropriate support of CPTs.
For tree topology density estimation, this can be done using the sample of tree topologies that is given as the data set.
For variational Bayesian phylogenetic inference, as no sample of tree topologies is available, one often resorts to fast bootstrap or MCMC methods \citep{Minh2013UltrafastAF, Zhang2020ImprovedVB}.
Let $\mathbb{S}_{\mathrm{r}}$ denote the root subsplits and $\mathbb{S}_{\mathrm{ch|pa}}$ denotes the child-parent subsplit pairs in the support.
The parameters of SBNs are then $p=\{p_{s_1}; s_1\in \mathbb{S}_{\mathrm{r}}\}\cup \{p_{s|t}; s|t\in\mathbb{S}_{\mathrm{ch|pa}} \}$ where
\begin{equation}\label{eq:sbn_param}
 p_{s_1} = p(S_1=s_1),\quad p_{s|t} = p(S_i=s|S_{\pi_i}=t), \; \forall i>1. 
\end{equation}

\section{Theoretical Results}
\subsection{Proof of Theorem \ref{thm:sivi}}\label{app:proof-sivi}
The asymptotically unbiasedness is a direct result of the strong law of large numbers. To prove $L^{K,J}(\bm{\phi},\bm{\psi}) \leq L^{K,J+1}(\bm{\phi},\bm{\psi}) \leq L^{K}(\bm{\phi},\bm{\psi}),\forall J$, we have three steps as follows.

\paragraph{Step 1} As the first step, we will give alternative expressions for $L^{K,J}(\bm{\phi},\bm{\psi})$ and $L^{K}(\bm{\phi},\bm{\psi})$.
Let $Q_{\bm{\psi}}^{J}(\bm{q}|\tau^k,\bm{z}^{k,0:J}) = \frac{1}{J+1}\sum_{j=0}^JQ_{\bm{\psi}}(\bm{q}|\tau^k,\bm{z}^{k,j})$.
By symmetry, we have
\begin{align*}
& L^{K,J}(\bm{\phi},\bm{\psi})\\
=& \frac{1}{J+1}\sum_{j=0}^J \mathbb{E}_{\left\langle(\tau^k,\bm{q}^k, \bm{z}^{k,j})\sim Q_{\bm{\phi},\bm{\psi}}(\tau,\bm{q},\bm{z})\right\rangle_{k=1}^K} \mathbb{E}_{\left\langle\bm{z}^{k,(0:J)\backslash j}\sim Q_{\bm{\psi}}(\bm{z}|\tau^k)\right\rangle_{k=1}^K}\log\left(\frac{1}{K}\sum_{k=1}^K\frac{P(\bm{Y}|\tau^k,\bm{q}^k)P(\tau^k,\bm{q}^k)}{Q_{\bm{\phi}}(\tau^k)\frac{1}{J+1}\sum_{j=0}^JQ_{\bm{\psi}}(\bm{q}^k|\tau^k,\bm{z}^{k,j})}\right)\\
=& \mathbb{E}_{\left\langle\tau^{k}\sim Q_{\bm{\phi}}(\tau)\right\rangle_{k=1}^K}
\mathbb{E}_{\left\langle\bm{z}^{k,0:J}\sim Q_{\bm{\psi}}(\bm{z}|\tau^k)\right\rangle_{k=1}^K}
\mathbb{E}_{\left\langle \bm{q}^{k}\sim Q_{\bm{\psi}}^{J}(\bm{q}|\tau^{k},\bm{z}^{k,0:J})\right\rangle_{k=1}^K} 
\log\left(\frac{1}{K}\sum_{k=1}^K\frac{P(\bm{Y}|\tau^k,\bm{q}^k)P(\tau^k,\bm{q}^k)}{Q_{\bm{\phi}}(\tau^k)Q_{\bm{\psi}}^{J}(\bm{q}^k|\tau^k,\bm{z}^{k,0:J})}\right).
\end{align*}    
where $(0:J)\backslash j=\{0,\ldots,j-1\}\cup\{j+1,\ldots,J\}$.
Using the fact that 
\[
\mathbb{E}_{\bm{z}^{k,0:J}\sim Q_{\bm{\psi}}(\bm{z}|\tau^k)}Q_{\bm{\psi}}^{J}(\bm{q}|\tau^k,\bm{z}^{k,0:J}) = Q_{\bm{\psi}}(\bm{q}|\tau^k),\quad k=1,\ldots,K,
\]
we can rewrite $L^K(\bm{\phi},\bm{\psi})$ as
\[
L^K(\bm{\phi},\bm{\psi}) = \mathbb{E}_{\left\langle\tau^{k}\sim Q_{\bm{\phi}}(\tau)\right\rangle_{k=1}^K}
\mathbb{E}_{\left\langle\bm{z}^{k,0:J}\sim Q_{\bm{\psi}}(\bm{z}|\tau^k)\right\rangle_{k=1}^K}
\mathbb{E}_{\left\langle \bm{q}^{k}\sim Q_{\bm{\psi}}^{J}(\bm{q}|\tau^{k},\bm{z}^{k,0:J})\right\rangle_{k=1}^K} 
\log\left(\frac{1}{K}\sum_{k=1}^K\frac{P(\bm{Y}|\tau^k,\bm{q}^k)P(\tau^k,\bm{q}^k)}{Q_{\bm{\phi}}(\tau^k)Q_{\bm{\psi}}(\bm{q}^k|\tau^k)}\right).
\]
In this way, $L^{K,J}(\bm{\phi},\bm{\psi})$ and $L^{K}(\bm{\phi},\bm{\psi})$ share the same reference distribution for expectation.

\paragraph{Step 2} Let $Q_{\bm{\phi},\bm{\psi}}^J(\tau^{1:K},\bm{q}^{1:K},\bm{z}^{1:K,0:J})=\prod_{k=1}^KQ_{\bm{\psi}}^{J}(\bm{q}^k|\tau^k,\bm{z}^{k,0:J})Q_{\bm{\psi}}(\bm{z}^{k,0:J}|\tau^k)Q_{\bm{\phi}}(\tau^k)$. We will show that the following two functions are both probability density functions:
\[
\left\{
\begin{array}{rcl}
f_{\bm{\phi},\bm{\psi}}^J(\tau^{1:K},\bm{q}^{1:K},\bm{z}^{1:K,0:J}) &=& \frac{
\sum_{k=1}^K\frac{P(\bm{Y}|\tau^k,\bm{q}^k)P(\tau^k,\bm{q}^k)}{Q_{\bm{\phi}}(\tau^k)Q_{\bm{\psi}}^{J}(\bm{q}^k|\tau^k,\bm{z}^{k,0:J})}
}{
\sum_{k=1}^K\frac{P(\bm{Y}|\tau^k,\bm{q}^k)P(\tau^k,\bm{q}^k)}{Q_{\bm{\phi}}(\tau^k)Q_{\bm{\psi}}(\bm{q}^k|\tau^k)}
}
Q_{\bm{\phi},\bm{\psi}}^J (\tau^{1:K},\bm{q}^{1:K},\bm{z}^{1:K,0:J});  \\
h_{\bm{\phi},\bm{\psi}}^J(\tau^{1:K},\bm{q}^{1:K},\bm{z}^{1:K,0:J+1}) &=& \frac{\sum_{k=1}^K\frac{P(\bm{Y}|\tau^k,\bm{q}^k)P(\tau^k,\bm{q}^k)}{Q_{\bm{\phi}}(\tau^k)Q_{\bm{\psi}}^{J}(\bm{q}^k|\tau^k,\bm{z}^{k,0:J})}}{\sum_{k=1}^K\frac{P(\bm{Y}|\tau^k,\bm{q}^k)P(\tau^k,\bm{q}^k)}{Q_{\bm{\phi}}(\tau^k)Q_{\bm{\psi}}^{J+1}(\bm{q}^k|\tau^k,\bm{z}^{k,0:J+1})}}
Q_{\bm{\phi},\bm{\psi}}^{J+1}(\tau^{1:K},\bm{q}^{1:K},\bm{z}^{1:K,0:J+1}).
\end{array}
\right.
\]

To prove $f_{\bm{\phi},\bm{\psi}}^J$ is a probability density function, we first integrate out $\bm{z}^{1:K,0:J}$, i.e.
\begin{align*}
&\int f_{\bm{\phi},\bm{\psi}}^J(\tau^{1:K},\bm{q}^{1:K},\bm{z}^{1:K,0:J})\;d\bm{z}^{1:K,0:J} \\
=& \frac{1}{
\sum_{k=1}^K\frac{P(\bm{Y}|\tau^k,\bm{q}^k)P(\tau^k,\bm{q}^k)}{Q_{\bm{\phi}}(\tau^k)Q_{\bm{\psi}}(\bm{q}^k|\tau^k)}
}
\cdot\sum_{k=1}^KP(\bm{Y}|\tau^k,\bm{q}^k)P(\tau^k,\bm{q}^k)\int \left[\prod_{l\neq k}Q_{\bm{\phi}}(\tau^l)Q_{\bm{\psi}}^{J}(\bm{q}^l|\tau^l,\bm{z}^{l,0:J})\right]\prod_{l=1}^KQ_{\bm{\psi}}(\bm{z}^{l,0:J}|\tau^l)\;d\bm{z}^{1:K,0:J}\\
=&\frac{\sum_{k=1}^KP(\bm{Y}|\tau^k,\bm{q}^k)P(\tau^k,\bm{q}^k)\prod_{l\neq k}Q_{\bm{\phi}}(\tau^l)Q_{\bm{\psi}}(\bm{q}^l|\tau^l)}{\sum_{k=1}^K\frac{P(\bm{Y}|\tau^k,\bm{q}^k)P(\tau^k,\bm{q}^k)}{Q_{\bm{\phi}}(\tau^k)Q_{\bm{\psi}}(\bm{q}^k|\tau^k)}}.
\end{align*}
Noting that
\[
\sum_{k=1}^KP(\bm{Y}|\tau^k,\bm{q}^k)P(\tau^k,\bm{q}^k)\prod_{l\neq k}Q_{\bm{\phi}}(\tau^l)Q_{\bm{\psi}}(\bm{q}^l|\tau^l) = \sum_{k=1}^K \frac{P(\bm{Y}|\tau^k,\bm{q}^k)P(\tau^k,\bm{q}^k)}{Q_{\bm{\phi}}(\tau^k)Q_{\bm{\psi}}(\bm{q}^k|\tau^k)} \cdot \prod_{l=1}^KQ_{\bm{\phi}}(\tau^l)Q_{\bm{\psi}}(\bm{q}^l|\tau^l),
\]
we therefore have
\[
\int f_{\bm{\phi},\bm{\psi}}^J(\tau^{1:K},\bm{q}^{1:K},\bm{z}^{1:K,0:J})\;d\bm{z}^{1:K,0:J} = \prod_{l=1}^KQ_{\bm{\phi}}(\tau^l)Q_{\bm{\psi}}(\bm{q}^l|\tau^l),
\]
which is clearly a density function of $\tau^{1:K}$ and $\bm{q}^{1:K}$.
    
To prove $h_{\bm{\phi},\bm{\psi}}^J$ is a probability density function,
it suffices to show
    \[
\mathbb{E}_{\left\langle\tau^{k}\sim Q_{\bm{\phi}}(\tau)\right\rangle_{k=1}^K}
\mathbb{E}_{\left\langle\bm{z}^{k,0:J+1}\sim Q_{\bm{\psi}}(\bm{z}|\tau^k)\right\rangle_{k=1}^K}
\mathbb{E}_{\left\langle \bm{q}^{k}\sim Q_{\bm{\psi}}^{J}(\bm{q}|\tau^{k},\bm{z}^{k,0:J+1})\right\rangle_{k=1}^K} 
\frac{\sum_{k=1}^K\frac{P(\bm{Y}|\tau^k,\bm{q}^k)P(\tau^k,\bm{q}^k)}{Q_{\bm{\phi}}(\tau^k)Q_{\bm{\psi}}^{J}(\bm{q}^k|\tau^k,\bm{z}^{k,0:J})}}{\sum_{k=1}^K\frac{P(\bm{Y}|\tau^k,\bm{q}^k)P(\tau^k,\bm{q}^k)}{Q_{\bm{\phi}}(\tau^k)Q_{\bm{\psi}}^{J+1}(\bm{q}^k|\tau^k,\bm{z}^{k,0:J+1})}} = 1.
\]
Let $\{I_k: I_k\subset\{0,\ldots,J+1\}, |I_k|=J+1, k=1,\ldots, K\}$ be uniformly distributed subsets with distinct indices from $\{0,\ldots,J+1\}$. Let $Q_{\bm{\psi}}^{J}(\bm{q}|\tau^k,\bm{z}^{k,I_k}) = \frac{1}{J+1}\sum_{j\in I_k}Q_{\bm{\psi}}(\bm{q}|\tau^k,\bm{z}^{k,j})$. By symmetry, we have
\begin{align*}
&\mathbb{E}_{\left\langle\bm{z}^{k,0:J+1}\sim Q_{\bm{\psi}}(\bm{z}|\tau^k)\right\rangle_{k=1}^K}
\mathbb{E}_{\left\langle \bm{q}^{k}\sim Q_{\bm{\psi}}^{J}(\bm{q}|\tau^{k},\bm{z}^{k,0:J+1})\right\rangle_{k=1}^K} 
\frac{\sum_{k=1}^K\frac{P(\bm{Y}|\tau^k,\bm{q}^k)P(\tau^k,\bm{q}^k)}{Q_{\bm{\phi}}(\tau^k)Q_{\bm{\psi}}^{J}(\bm{q}^k|\tau^k,\bm{z}^{k,0:J})}}{\sum_{k=1}^K\frac{P(\bm{Y}|\tau^k,\bm{q}^k)P(\tau^k,\bm{q}^k)}{Q_{\bm{\phi}}(\tau^k)Q_{\bm{\psi}}^{J+1}(\bm{q}^k|\tau^k,\bm{z}^{k,0:J+1})}} \\
=&\mathbb{E}_{\left\langle\bm{z}^{k,0:J+1}\sim Q_{\bm{\psi}}(\bm{z}|\tau^k)\right\rangle_{k=1}^K}\mathbb{E}_{I_{1:K}}\mathbb{E}_{\left\langle \bm{q}^{k}\sim Q_{\bm{\psi}}^{J}(\bm{q}|\tau^k,\bm{z}^{k,I_k})\right\rangle_{k=1}^K}
\frac{\sum_{k=1}^K\frac{P(\bm{Y}|\tau^k,\bm{q}^k)P(\tau^k,\bm{q}^k)}{Q_{\bm{\phi}}(\tau^k)Q_{\bm{\psi}}^{J}(\bm{q}^k|\tau^k,\bm{z}^{k,0:J})}}{\sum_{k=1}^K\frac{P(\bm{Y}|\tau^k,\bm{q}^k)P(\tau^k,\bm{q}^k)}{Q_{\bm{\phi}}(\tau^k)Q_{\bm{\psi}}^{J+1}(\bm{q}^k|\tau^k,\bm{z}^{k,0:J+1})}}\\
=& \mathbb{E}_{\left\langle\bm{z}^{k,0:J+1}\sim Q_{\bm{\psi}}(\bm{z}|\tau^k)\right\rangle_{k=1}^K}\mathbb{E}_{I_{1:K}}
\left(\int\frac{\sum_{k=1}^K\frac{P(\bm{Y}|\tau^k,\bm{q}^k)P(\tau^k,\bm{q}^k)}{Q_{\bm{\phi}}(\tau^k)}\prod_{l\neq k}Q_{\bm{\psi}}^{J}(\bm{q}^l|\tau^l,\bm{z}^{l,I_l})}{\sum_{k=1}^K\frac{P(\bm{Y}|\tau^k,\bm{q}^k)P(\tau^k,\bm{q}^k)}{Q_{\bm{\phi}}(\tau^k)Q_{\bm{\psi}}^{J+1}(\bm{q}^k|\tau^k,\bm{z}^{k,0:J+1})}}\;d\bm{q}^{1:K}\right)\\
=& \mathbb{E}_{\left\langle\bm{z}^{k,0:J+1}\sim Q_{\bm{\psi}}(\bm{z}|\tau^k)\right\rangle_{k=1}^K}
\left(\int\frac{\sum_{k=1}^K\frac{P(\bm{Y}|\tau^k,\bm{q}^k)P(\tau^k,\bm{q}^k)}{Q_{\bm{\phi}}(\tau^k)}\prod_{l\neq k}Q_{\bm{\psi}}^{J+1}(\bm{q}^l|\tau^l,\bm{z}^{l,0:J+1})}{\sum_{k=1}^K\frac{P(\bm{Y}|\tau^k,\bm{q}^k)P(\tau^k,\bm{q}^k)}{Q_{\bm{\phi}}(\tau^k)Q_{\bm{\psi}}^{J+1}(\bm{q}^k|\tau^k,\bm{z}^{k,0:J+1})}}\;d\bm{q}^{1:K}\right)\\
=&\mathbb{E}_{\left\langle\bm{z}^{k,0:J+1}\sim Q_{\bm{\psi}}(\bm{z}|\tau^k)\right\rangle_{k=1}^K}
\int \prod_{l=1}^K Q_{\bm{\psi}}^{J+1}(\bm{q}^l|\tau^l,\bm{z}^{l,0:J+1})\;d\bm{q}^{1:K}.\\
=&1.
\end{align*}
Here, we use the fact that
\[
E_{I_l}Q_{\bm{\psi}}^{J}(\bm{q}^l|\tau^l,\bm{z}^{l,I_l}) = Q_{\bm{\psi}}^{J+1}(\bm{q}^l|\tau^l,\bm{z}^{l,0:J+1}),\quad \forall\; l=1,\ldots,K.
\]

\paragraph{Step 3} Now, we are ready to prove that $L^{K,J}(\bm{\phi},\bm{\psi}) \leq L^{K,J+1}(\bm{\phi},\bm{\psi}) \leq L^{K}(\bm{\phi},\bm{\psi}),\forall J$. The gap between $L^K$ and $L^{K,J}$ can be expressed as
\begin{align*}
&L^K(\bm{\phi},\bm{\psi}) - L^{K,J}(\bm{\phi},\bm{\psi}) \\
= & \mathbb{E}_{\left\langle\tau^{k}\sim Q_{\bm{\phi}}(\tau)\right\rangle_{k=1}^K}
\mathbb{E}_{\left\langle\bm{z}^{k,0:J}\sim Q_{\bm{\psi}}(\bm{z}|\tau^k)\right\rangle_{k=1}^K}
\mathbb{E}_{\left\langle \bm{q}^{k}\sim Q_{\bm{\psi}}^{J}(\bm{q}|\tau^{k},\bm{z}^{k,0:J})\right\rangle_{k=1}^K} 
\log\left(\frac{\sum_{k=1}^K\frac{P(\bm{Y}|\tau^k,\bm{q}^k)P(\tau^k,\bm{q}^k)}{Q_{\bm{\phi}}(\tau^k)Q_{\bm{\psi}}(\bm{q}^k|\tau^k)}}{\sum_{k=1}^K\frac{P(\bm{Y}|\tau^k,\bm{q}^k)P(\tau^k,\bm{q}^k)}{Q_{\bm{\phi}}(\tau^k)Q_{\bm{\psi}}^{J}(\bm{q}^k|\tau^k,\bm{z}^{k,0:J})}}\right).\\
= & \mathbb{E}_{\left\langle\tau^{k}\sim Q_{\bm{\phi}}(\tau)\right\rangle_{k=1}^K}
\mathbb{E}_{\left\langle\bm{z}^{k,0:J}\sim Q_{\bm{\psi}}(\bm{z}|\tau^k)\right\rangle_{k=1}^K}
\mathbb{E}_{\left\langle \bm{q}^{k}\sim Q_{\bm{\psi}}^{J}(\bm{q}|\tau^{k},\bm{z}^{k,0:J})\right\rangle_{k=1}^K}
\log\left(\frac{Q_{\bm{\phi},\bm{\psi}}^J(\tau^{1:K},\bm{q}^{1:K},\bm{z}^{1:K,0:J})}{f_{\bm{\phi},\bm{\psi}}^J(\tau^{1:K},\bm{q}^{1:K},\bm{z}^{1:K,0:J})}\right) \\
= & \mathrm{KL}\left(Q_{\bm{\phi},\bm{\psi}}^J(\tau^{1:K},\bm{q}^{1:K},\bm{z}^{1:K,0:J})\| f_{\bm{\phi},\bm{\psi}}^J(\tau^{1:K},\bm{q}^{1:K},\bm{z}^{1:K,0:J})\right)
\end{align*}
This proves that $L^{K,J}(\bm{\phi},\bm{\psi})\leq L^{K}(\bm{\phi},\bm{\psi})$.
Using a similar argument, 
\begin{align*}
&L^{K,J+1}(\bm{\phi},\bm{\psi}) - L^{K,J}(\bm{\phi},\bm{\psi}) \\
= & \mathbb{E}_{\left\langle\tau^{k}\sim Q_{\bm{\phi}}(\tau)\right\rangle_{k=1}^K}
\mathbb{E}_{\left\langle\bm{z}^{k,0:J+1}\sim Q_{\bm{\psi}}(\bm{z}|\tau^k)\right\rangle_{k=1}^K}
\mathbb{E}_{\left\langle \bm{q}^{k}\sim Q_{\bm{\psi}}^{J}(\bm{q}|\tau^{k},\bm{z}^{k,0:J+1})\right\rangle_{k=1}^K} 
\log\left(\frac{\sum_{k=1}^K\frac{P(\bm{Y}|\tau^k,\bm{q}^k)P(\tau^k,\bm{q}^k)}{Q_{\bm{\phi}}(\tau^k)Q_{\bm{\psi}}^{J+1}(\bm{q}^k|\tau^k,\bm{z}^{k,0:J})}}{\sum_{k=1}^K\frac{P(\bm{Y}|\tau^k,\bm{q}^k)P(\tau^k,\bm{q}^k)}{Q_{\bm{\phi}}(\tau^k)Q_{\bm{\psi}}^{J}(\bm{q}^k|\tau^k,\bm{z}^{k,0:J})}}\right).\\
= & \mathbb{E}_{\left\langle\tau^{k}\sim Q_{\bm{\phi}}(\tau)\right\rangle_{k=1}^K}
\mathbb{E}_{\left\langle\bm{z}^{k,0:J+1}\sim Q_{\bm{\psi}}(\bm{z}|\tau^k)\right\rangle_{k=1}^K}
\mathbb{E}_{\left\langle \bm{q}^{k}\sim Q_{\bm{\psi}}^{J}(\bm{q}|\tau^{k},\bm{z}^{k,0:J+1})\right\rangle_{k=1}^K} 
\log\left(\frac{Q_{\bm{\phi},\bm{\psi}}^{J+1}(\tau^{1:K},\bm{q}^{1:K},\bm{z}^{1:K,0:J+1})}{h_{\bm{\phi},\bm{\psi}}^{J}(\tau^{1:K},\bm{q}^{1:K},\bm{z}^{1:K,0:J+1})}\right).\\
= & \mathrm{KL}\left(Q_{\bm{\phi},\bm{\psi}}^{J+1}(\tau^{1:K},\bm{q}^{1:K},\bm{z}^{1:K,0:J+1})\| h_{\bm{\phi},\bm{\psi}}^{J}(\tau^{1:K},\bm{q}^{1:K},\bm{z}^{1:K,0:J+1})\right).
\end{align*}
This proves that $L^{K,J}(\bm{\phi},\bm{\psi})\leq L^{K,J+1}(\bm{\phi},\bm{\psi})$.
\qed

\subsection{Proof of Theorem \ref{thm:iwhvi}}
We will prove Theorem \ref{thm:iwhvi} following a similar three steps procedure as in the Appendix \ref{app:proof-sivi}. Note that  asymptotically unbiasedness of $L^{K,J}_w(\bm{\phi},\bm{\psi},\bm{\xi})$ is still a direct result of the strong law of large numbers.
\paragraph{Step 1} We first derive alternative expressions for $L^{K,J}_w(\bm{\phi},\bm{\psi},\bm{\xi})$ and $L^{K}(\bm{\phi},\bm{\psi})$.
Let $H^J_{\bm{\psi},\bm{\xi}}(\bm{q}^k,\bm{z}^{k,0:J}|\tau^k)=\frac{1}{J+1}\sum_{j=0}^J \frac{Q_{\bm{\psi}}(\bm{z}^{k,j}|\tau^k)Q_{\bm{\psi}}(\bm{q}^k|\tau^k,\bm{z}^{k,j})}{R_{\bm{\xi}}(\bm{z}^{k,j}|\tau^k,\bm{q}^k)}$ and 
\[
Q_{\bm{\phi},\bm{\psi},\bm{\xi}}^{J}(\tau^{1:K},\bm{q}^{1:K},\bm{z}^{1:K,0:J}) = \prod_{k=1}^K H^J_{\bm{\psi},\bm{\xi}}(\bm{q}^k,\bm{z}^{k,0:J}|\tau^k) R_{\bm{\xi}}(\bm{z}^{k,0:J}|\tau^k,\bm{q}^k)Q_{\bm{\phi}}(\tau^k).
\]
Note that $Q_{\bm{\phi},\bm{\psi},\bm{\xi}}^{J}(\tau^{1:K},\bm{q}^{1:K},\bm{z}^{1:K,0:J})$ is indeed a valid proability density function.
By symmetry,
\begin{align*}
& L^{K,J}_w(\bm{\phi},\bm{\psi},\bm{\xi})\\
=& \frac{1}{J+1}\sum_{j=0}^J \mathbb{E}_{\left\langle(\tau^k,\bm{q}^k, \bm{z}^{k,j})\sim Q_{\bm{\phi},\bm{\psi}}(\tau,\bm{q},\bm{z})\right\rangle_{k=1}^K} \mathbb{E}_{\left\langle\bm{z}^{k,(0:J)\backslash j}\sim Q_{\bm{\psi}}(\bm{z}|\tau^k,\bm{q}^k)\right\rangle_{k=1}^K}\log\left(\frac{1}{K}\sum_{k=1}^K\frac{P(\bm{Y}|\tau^k,\bm{q}^k)P(\tau^k,\bm{q}^k)}{Q_{\bm{\phi}}(\tau^k)H^J_{\bm{\psi},\bm{\xi}}(\bm{q}^k,\bm{z}^{k,0:J}|\tau^k)}\right)\\
=& \mathbb{E}_{(\tau^{1:K},\bm{q}^{1:K},\bm{z}^{1:K,0:J})\sim Q_{\bm{\phi},\bm{\psi},\bm{\xi}}^{J}(\tau^{1:K},\bm{q}^{1:K},\bm{z}^{1:K,0:J})}
\log\left(\frac{1}{K}\sum_{k=1}^K\frac{P(\bm{Y}|\tau^k,\bm{q}^k)P(\tau^k,\bm{q}^k)}{Q_{\bm{\phi}}(\tau^k)H^J_{\bm{\psi},\bm{\xi}}(\bm{q}^k,\bm{z}^{k,0:J}|\tau^k)}\right).
\end{align*}  
Using the fact that 
\[
\int Q_{\bm{\phi},\bm{\psi},\bm{\xi}}^{J}(\tau^{1:K},\bm{q}^{1:K},\bm{z}^{1:K,0:J})\; d \bm{z}^{1:K,0:J}= Q_{\bm{\psi}}(\bm{q}^{1:K},\tau^{1:K})
\]
we can rewrite $L^K(\bm{\phi},\bm{\psi})$ as
\[
L^K(\bm{\phi},\bm{\psi}) = \mathbb{E}_{(\tau^{1:K},\bm{q}^{1:K},\bm{z}^{1:K,0:J})\sim Q_{\bm{\phi},\bm{\psi},\bm{\xi}}^{J}(\tau^{1:K},\bm{q}^{1:K},\bm{z}^{1:K,0:J})}
\log\left(\frac{1}{K}\sum_{k=1}^K\frac{P(\bm{Y}|\tau^k,\bm{q}^k)P(\tau^k,\bm{q}^k)}{Q_{\bm{\phi}}(\tau^k)Q_{\bm{\psi}}(\bm{q}^k|\tau^k)}\right).
\]
Therefore, the $L^{K,J}_w(\bm{\phi},\bm{\psi},\bm{\xi})$ and $L^K(\bm{\phi},\bm{\psi})$ share the same reference distribution in expectation, as in Appendix \ref{app:proof-sivi}.

\paragraph{Step 2}
Next, we will show the following two functions are both probability density functions:
\begin{equation*}
\left\{
\begin{array}{rcl}
f_{\bm{\phi},\bm{\psi},\bm{\xi}}^J(\tau^{1:K},\bm{q}^{1:K},\bm{z}^{1:K,0:J}) &=& \frac{\sum_{k=1}^K\frac{P(\bm{Y}|\tau^k,\bm{q}^k)P(\tau^k,\bm{q}^k)}{Q_{\bm{\phi}}(\tau^k)H^J_{\bm{\psi},\bm{\xi}}(\bm{q}^k,\bm{z}^{k,0:J}|\tau^k)}}{\sum_{k=1}^K\frac{P(\bm{Y}|\tau^k,\bm{q}^k)P(\tau^k,\bm{q}^k)}{Q_{\bm{\phi}}(\tau^k)Q_{\bm{\psi}}(\bm{q}^k|\tau^k)}} Q_{\bm{\phi},\bm{\psi},\bm{\xi}}^{J}(\tau^{1:K},\bm{q}^{1:K},\bm{z}^{1:K,0:J});\\
h_{\bm{\phi},\bm{\psi},\bm{\xi}}^J(\tau^{1:K},\bm{q}^{1:K},\bm{z}^{1:K,0:J+1}) & = &\frac{\sum_{k=1}^K\frac{P(\bm{Y}|\tau^k,\bm{q}^k)P(\tau^k,\bm{q}^k)}{Q_{\bm{\phi}}(\tau^k)H^J_{\bm{\psi},\bm{\xi}}(\bm{q}^k,\bm{z}^{k,0:J}|\tau^k)}}{\sum_{k=1}^K\frac{P(\bm{Y}|\tau^k,\bm{q}^k)P(\tau^k,\bm{q}^k)}{Q_{\bm{\phi}}(\tau^k)H^{J+1}_{\bm{\psi},\bm{\xi}}(\bm{q}^k,\bm{z}^{k,0:J+1}|\tau^k)}} Q_{\bm{\phi},\bm{\psi},\bm{\xi}}^{J+1}(\tau^{1:K},\bm{q}^{1:K},\bm{z}^{1:K,0:J+1})\\
\end{array}
\right.
\end{equation*}
Integrating out $\bm{z}^{1:K,0:J}$ in $f_{\bm{\phi},\bm{\psi},\bm{\xi}}^J(\tau^{1:K},\bm{q}^{1:K},\bm{z}^{1:K,0:J}) $ yields
\begin{align*}
&\int f_{\bm{\phi},\bm{\psi},\bm{\xi}}^J(\tau^{1:K},\bm{q}^{1:K},\bm{z}^{1:K,0:J})\;d\bm{z}^{1:K,0:J} \\
=& \frac{1}{
\sum_{k=1}^K\frac{P(\bm{Y}|\tau^k,\bm{q}^k)P(\tau^k,\bm{q}^k)}{Q_{\bm{\phi}}(\tau^k)Q_{\bm{\psi}}(\bm{q}^k|\tau^k)}
}
\cdot\sum_{k=1}^KP(\bm{Y}|\tau^k,\bm{q}^k)P(\tau^k,\bm{q}^k)\int \left[\prod_{l\neq k}H^J_{\bm{\psi},\bm{\xi}}(\bm{q}^l,\bm{z}^{l,0:J}|\tau^l) Q_{\bm{\phi}}(\tau^l)\right]\prod_{l=1}^K R_{\bm{\xi}}(\bm{z}^{l,0:J}|\tau^l,\bm{q}^l)\;d\bm{z}^{1:K,0:J}\\
=&\frac{\sum_{k=1}^KP(\bm{Y}|\tau^k,\bm{q}^k)P(\tau^k,\bm{q}^k)\prod_{l\neq k}Q_{\bm{\phi}}(\tau^l)Q_{\bm{\psi}}(\bm{q}^l|\tau^l)}{\sum_{k=1}^K\frac{P(\bm{Y}|\tau^k,\bm{q}^k)P(\tau^k,\bm{q}^k)}{Q_{\bm{\phi}}(\tau^k)Q_{\bm{\psi}}(\bm{q}^k|\tau^k)}}\\
=&\frac{\sum_{k=1}^K \frac{P(\bm{Y}|\tau^k,\bm{q}^k)P(\tau^k,\bm{q}^k)}{Q_{\bm{\phi}}(\tau^k)Q_{\bm{\psi}}(\bm{q}^k|\tau^k)} \cdot \prod_{l=1}^KQ_{\bm{\phi}}(\tau^l)Q_{\bm{\psi}}(\bm{q}^l|\tau^l)}{\sum_{k=1}^K\frac{P(\bm{Y}|\tau^k,\bm{q}^k)P(\tau^k,\bm{q}^k)}{Q_{\bm{\phi}}(\tau^k)Q_{\bm{\psi}}(\bm{q}^k|\tau^k)}}\\
=&\prod_{l=1}^KQ_{\bm{\phi}}(\tau^l)Q_{\bm{\psi}}(\bm{q}^l|\tau^l)
\end{align*}
which just the joint variational distribution of $(\tau^{1:K},\bm{q}^{1:K})$. Therefore, $f_{\bm{\phi},\bm{\psi},\bm{\xi}}^J(\tau^{1:K},\bm{q}^{1:K},\bm{z}^{1:K,0:J})$ is a valid probability density function.

To prove $h_{\bm{\phi},\bm{\psi},\bm{\xi}}^J(\tau^{1:K},\bm{q}^{1:K},\bm{z}^{1:K,0:J+1})$ is a valid probability density function, it suffices to show
\[
\mathbb{E}_{Q_{\bm{\phi},\bm{\psi},\bm{\xi}}^{J+1}(\tau^{1:K},\bm{q}^{1:K},\bm{z}^{1:K,0:J+1})}\frac{\sum_{k=1}^K\frac{P(\bm{Y}|\tau^k,\bm{q}^k)P(\tau^k,\bm{q}^k)}{Q_{\bm{\phi}}(\tau^k)H^J_{\bm{\psi},\bm{\xi}}(\bm{q}^k,\bm{z}^{k,0:J}|\tau^k)}}{\sum_{k=1}^K\frac{P(\bm{Y}|\tau^k,\bm{q}^k)P(\tau^k,\bm{q}^k)}{Q_{\bm{\phi}}(\tau^k)H^{J+1}_{\bm{\psi},\bm{\xi}}(\bm{q}^k,\bm{z}^{k,0:J+1}|\tau^k)}} = 1.
\]
Let $\{I_k: I_k\subset\{0,\ldots,J+1\}, |I_k|=J+1, k=1,\ldots, K\}$ be uniformly distributed subsets with distinct indices from $\{0,\ldots,J+1\}$. 
Let $H^{J}_{\bm{\psi},\bm{\xi}}(\bm{q}^k,\bm{z}^{k,I_k}|\tau^k)=\frac{1}{J+1}\sum_{j\in I_k} \frac{Q_{\bm{\psi}}(\bm{z}^{k,j}|\tau^k)Q_{\bm{\psi}}(\bm{q}^k|\tau^k,\bm{z}^{k,j})}{R_{\bm{\xi}}(\bm{z}^{k,j}|\tau^k,\bm{q}^k)}$
and 
\[
Q_{\bm{\phi},\bm{\psi},\bm{\xi}}^{J}(\tau^{1:K},\bm{q}^{1:K},\bm{z}^{1:K,I_{1:K}}) = \prod_{k=1}^K H^J_{\bm{\psi},\bm{\xi}}(\bm{q}^k,\bm{z}^{k,I_k}|\tau^k) R_{\bm{\xi}}(\bm{z}^{k,I_k}|\tau^k,\bm{q}^k)Q_{\bm{\phi}}(\tau^k).
\]
By symmetry, we have
\[
\mathbb{E}_{I_{1:K}} Q_{\bm{\phi},\bm{\psi},\bm{\xi}}^{J}(\tau^{1:K},\bm{q}^{1:K},\bm{z}^{1:K,I_{1:K}})\prod_{k=1}^K R_{\bm{\xi}}(\bm{z}^{k,-I_k}|\tau^k,\bm{q}^k)=Q_{\bm{\phi},\bm{\psi},\bm{\xi}}^{J+1}(\tau^{1:K},\bm{q}^{1:K},\bm{z}^{1:K,0:J+1})
\]
where $-I_k = (0:J+1)\backslash I_k$, and thus
\begin{align*}
& \mathbb{E}_{Q_{\bm{\phi},\bm{\psi},\bm{\xi}}^{J+1}(\tau^{1:K},\bm{q}^{1:K},\bm{z}^{1:K,0:J+1})}\frac{\sum_{k=1}^K\frac{P(\bm{Y}|\tau^k,\bm{q}^k)P(\tau^k,\bm{q}^k)}{Q_{\bm{\phi}}(\tau^k)H^J_{\bm{\psi},\bm{\xi}}(\bm{q}^k,\bm{z}^{k,0:J}|\tau^k)}}{\sum_{k=1}^K\frac{P(\bm{Y}|\tau^k,\bm{q}^k)P(\tau^k,\bm{q}^k)}{Q_{\bm{\phi}}(\tau^k)H^{J+1}_{\bm{\psi},\bm{\xi}}(\bm{q}^k,\bm{z}^{k,0:J+1}|\tau^k)}}\\
=& \mathbb{E}_{I_{1:K}} \mathbb{E}_{Q_{\bm{\phi},\bm{\psi},\bm{\xi}}^{J}(\tau^{1:K},\bm{q}^{1:K},\bm{z}^{1:K,I_{1:K}})}\frac{\sum_{k=1}^K\frac{P(\bm{Y}|\tau^k,\bm{q}^k)P(\tau^k,\bm{q}^k)}{Q_{\bm{\phi}}(\tau^k)H^J_{\bm{\psi},\bm{\xi}}(\bm{q}^k,\bm{z}^{k,I_k}|\tau^k)}\prod_{k=1}^K R_{\bm{\xi}}(\bm{z}^{k,-I_k}|\tau^k,\bm{q}^k)}{\sum_{k=1}^K\frac{P(\bm{Y}|\tau^k,\bm{q}^k)P(\tau^k,\bm{q}^k)}{Q_{\bm{\phi}}(\tau^k)H^{J+1}_{\bm{\psi},\bm{\xi}}(\bm{q}^k,\bm{z}^{k,0:J+1}|\tau^k)}}\\
=& \mathbb{E}_{Q_{\bm{\phi}}(\tau^{1:K})}\mathbb{E}_{I_{1:K}} \int\int\frac{\sum_{k=1}^K\frac{P(\bm{Y}|\tau^k,\bm{q}^k)P(\tau^k,\bm{q}^k)}{Q_{\bm{\phi}}(\tau^k)}R_{\bm{\xi}}(\bm{z}^{k,0:J+1}|\tau^k,\bm{q}^k)\prod_{l\neq k} H^J_{\bm{\psi},\bm{\xi}}(\bm{q}^l,\bm{z}^{l,I_l}|\tau^l) R_{\bm{\xi}}(\bm{z}^{l,0:J+1}|\tau^l,\bm{q}^l)}{\sum_{k=1}^K\frac{P(\bm{Y}|\tau^k,\bm{q}^k)P(\tau^k,\bm{q}^k)}{Q_{\bm{\phi}}(\tau^k)H^{J+1}_{\bm{\psi},\bm{\xi}}(\bm{q}^k,\bm{z}^{k,0:J+1}|\tau^k)}}\; d\bm{z}^{1:K,0:J+1}\; d\bm{q}^{1:K}\\
=& \mathbb{E}_{Q_{\bm{\phi}}(\tau^{1:K})} \int\int\frac{\sum_{k=1}^K\frac{P(\bm{Y}|\tau^k,\bm{q}^k)P(\tau^k,\bm{q}^k)}{Q_{\bm{\phi}}(\tau^k)}R_{\bm{\xi}}(\bm{z}^{k,0:J+1}|\tau^k,\bm{q}^k)\prod_{l\neq k} H^{J+1}_{\bm{\psi},\bm{\xi}}(\bm{q}^l,\bm{z}^{l,0:J+1}|\tau^l) R_{\bm{\xi}}(\bm{z}^{l,0:J+1}|\tau^l,\bm{q}^l)}{\sum_{k=1}^K\frac{P(\bm{Y}|\tau^k,\bm{q}^k)P(\tau^k,\bm{q}^k)}{Q_{\bm{\phi}}(\tau^k)H^{J+1}_{\bm{\psi},\bm{\xi}}(\bm{q}^k,\bm{z}^{k,0:J+1}|\tau^k)}}\; d\bm{z}^{1:K,0:J+1}\; d\bm{q}^{1:K}\\
=& \mathbb{E}_{Q_{\bm{\phi}}(\tau^{1:K})} \int\int\frac{\sum_{k=1}^K\frac{P(\bm{Y}|\tau^k,\bm{q}^k)P(\tau^k,\bm{q}^k)}{Q_{\bm{\phi}}(\tau^k)H^{J+1}_{\bm{\psi},\bm{\xi}}(\bm{q}^k,\bm{z}^{k,0:J+1}|\tau^k)}\prod_{l=1}^K  H^{J+1}_{\bm{\psi},\bm{\xi}}(\bm{q}^l,\bm{z}^{l,0:J+1}|\tau^l) R_{\bm{\xi}}(\bm{z}^{l,0:J+1}|\tau^l,\bm{q}^l)}{\sum_{k=1}^K\frac{P(\bm{Y}|\tau^k,\bm{q}^k)P(\tau^k,\bm{q}^k)}{Q_{\bm{\phi}}(\tau^k)H^{J+1}_{\bm{\psi},\bm{\xi}}(\bm{q}^k,\bm{z}^{k,0:J+1}|\tau^k)}}\; d\bm{z}^{1:K,0:J+1}\; d\bm{q}^{1:K}\\
=&\mathbb{E}_{Q_{\bm{\phi}}(\tau^{1:K})} \int\int \prod_{l=1}^K  H^{J+1}_{\bm{\psi},\bm{\xi}}(\bm{q}^l,\bm{z}^{l,0:J+1}|\tau^l) R_{\bm{\xi}}(\bm{z}^{l,0:J+1}|\tau^l,\bm{q}^l)\; d\bm{z}^{1:K,0:J+1}\; d\bm{q}^{1:K}\\
=&1.
\end{align*}
Therefore, $h_{\bm{\phi},\bm{\psi},\bm{\xi}}^J(\tau^{1:K},\bm{q}^{1:K},\bm{z}^{1:K,0:J+1})$ is a valid probability density function.

\paragraph{Step 3} Now, we are ready to prove that $L^{K,J}_w(\bm{\phi},\bm{\psi},\bm{\xi}) \leq L^{K,J+1}_w(\bm{\phi},\bm{\psi},\bm{\xi}) \leq L^{K}(\bm{\phi},\bm{\psi}),\quad \forall J$. The gap between $L^{K,J}_w(\bm{\phi},\bm{\psi},\bm{\xi})$ and $L^{K}(\bm{\phi},\bm{\psi})$ is
\begin{align*}
&L^K(\bm{\phi},\bm{\psi}) - L^{K,J}_w(\bm{\phi},\bm{\psi},\bm{\xi}) \\
= & \mathbb{E}_{(\tau^{1:K},\bm{q}^{1:K},\bm{z}^{1:K,0:J})\sim Q_{\bm{\phi},\bm{\psi},\bm{\xi}}^{J}(\tau^{1:K},\bm{q}^{1:K},\bm{z}^{1:K,0:J})}
\log\left(\frac{\sum_{k=1}^K\frac{P(\bm{Y}|\tau^k,\bm{q}^k)P(\tau^k,\bm{q}^k)}{Q_{\bm{\phi}}(\tau^k)Q_{\bm{\psi}}(\bm{q}^k|\tau^k)}}{\sum_{k=1}^K\frac{P(\bm{Y}|\tau^k,\bm{q}^k)P(\tau^k,\bm{q}^k)}{Q_{\bm{\phi}}(\tau^k)H^J_{\bm{\psi},\bm{\xi}}(\bm{q}^k,\bm{z}^{k,0:J}|\tau^k)}}\right).\\
= & \mathbb{E}_{(\tau^{1:K},\bm{q}^{1:K},\bm{z}^{1:K,0:J})\sim Q_{\bm{\phi},\bm{\psi},\bm{\xi}}^{J}(\tau^{1:K},\bm{q}^{1:K},\bm{z}^{1:K,0:J})}
\log\left(\frac{Q_{\bm{\phi},\bm{\psi},\bm{\xi}}^{J}(\tau^{1:K},\bm{q}^{1:K},\bm{z}^{1:K,0:J})}{f_{\bm{\phi},\bm{\psi},\bm{\xi}}^J(\tau^{1:K},\bm{q}^{1:K},\bm{z}^{1:K,0:J})}\right) \\
= & \mathrm{KL}\left(Q_{\bm{\phi},\bm{\psi},\bm{\xi}}^{J}(\tau^{1:K},\bm{q}^{1:K},\bm{z}^{1:K,0:J})\| f_{\bm{\phi},\bm{\psi},\bm{\xi}}^J(\tau^{1:K},\bm{q}^{1:K},\bm{z}^{1:K,0:J})\right).
\end{align*}
This proves $L^{K,J}_w(\bm{\phi},\bm{\psi},\bm{\xi})\leq L^K(\bm{\phi},\bm{\psi})$. 
The gap between $L^{K,J}_w(\bm{\phi},\bm{\psi},\bm{\xi})$ and $L^{K,J+1}_w(\bm{\phi},\bm{\psi},\bm{\xi})$ is
\begin{align*}
&L^{K,J+1}_w(\bm{\phi},\bm{\psi},\bm{\xi}) - L^{K,J}_w(\bm{\phi},\bm{\psi},\bm{\xi}) \\
= & \mathbb{E}_{(\tau^{1:K},\bm{q}^{1:K},\bm{z}^{1:K,0:J+1})\sim Q_{\bm{\phi},\bm{\psi},\bm{\xi}}^{J+1}(\tau^{1:K},\bm{q}^{1:K},\bm{z}^{1:K,0:J+1})}
\log\left(\frac{\sum_{k=1}^K\frac{P(\bm{Y}|\tau^k,\bm{q}^k)P(\tau^k,\bm{q}^k)}{Q_{\bm{\phi}}(\tau^k)H^{J+1}_{\bm{\psi},\bm{\xi}}(\bm{q}^k,\bm{z}^{k,0:J+1}|\tau^k)}}{\sum_{k=1}^K\frac{P(\bm{Y}|\tau^k,\bm{q}^k)P(\tau^k,\bm{q}^k)}{Q_{\bm{\phi}}(\tau^k)H^J_{\bm{\psi},\bm{\xi}}(\bm{q}^k,\bm{z}^{k,0:J}|\tau^k)}}\right).\\
= & \mathbb{E}_{(\tau^{1:K},\bm{q}^{1:K},\bm{z}^{1:K,0:J+1})\sim Q_{\bm{\phi},\bm{\psi},\bm{\xi}}^{J+1}(\tau^{1:K},\bm{q}^{1:K},\bm{z}^{1:K,0:J+1})}
\log\left(\frac{Q_{\bm{\phi},\bm{\psi},\bm{\xi}}^{J+1}(\tau^{1:K},\bm{q}^{1:K},\bm{z}^{1:K,0:J+1})}{h_{\bm{\phi},\bm{\psi},\bm{\xi}}^J(\tau^{1:K},\bm{q}^{1:K},\bm{z}^{1:K,0:J+1})}\right).\\
= & \mathrm{KL}\left(Q_{\bm{\phi},\bm{\psi},\bm{\xi}}^{J+1}(\tau^{1:K},\bm{q}^{1:K},\bm{z}^{1:K,0:J+1})\| h_{\bm{\phi},\bm{\psi},\bm{\xi}}^J(\tau^{1:K},\bm{q}^{1:K},\bm{z}^{1:K,0:J+1})\right).
\end{align*}
This proves $L^{K,J}_w(\bm{\phi},\bm{\psi},\bm{\xi})\leq L^{K,J+1}_w(\bm{\phi},\bm{\psi},\bm{\xi})$. 

\end{document}